\RequirePackage{fix-cm}
\documentclass{article}

\usepackage{xcolor}
\definecolor{PaperAccent}{HTML}{F98E71}

\PassOptionsToPackage{numbers, compress}{natbib}
\usepackage[main, final]{neurips_2026}

\usepackage[utf8]{inputenc} % allow utf-8 input
\usepackage[T1]{fontenc}    % use 8-bit T1 fonts
\usepackage{hyperref}       % hyperlinks
\usepackage{url}            % simple URL typesetting
\usepackage{booktabs}       % professional-quality tables
\usepackage{amsfonts}       % blackboard math symbols
\usepackage{nicefrac}       % compact symbols for 1/2, etc.
\usepackage{microtype}      % microtypography

\title{Trust Region Policy Distillation}

\usepackage{amsmath}
\usepackage{amssymb}
\usepackage{mathtools}
\usepackage{amsthm}
\usepackage{graphicx}
\usepackage{colortbl}
\usepackage{array}
\usepackage{wrapfig}
\usepackage{algorithm}
\usepackage{algorithmic}
\usepackage{multirow}

\theoremstyle{plain}
\newtheorem{theorem}{Theorem}[section]

\theoremstyle{definition}
\newtheorem{definition}[theorem]{Definition}
\newtheorem{assumption}[theorem]{Assumption}
\theoremstyle{remark}
\newtheorem{remark}[theorem]{Remark}

\graphicspath{{figures/}}

\definecolor{0}{gray}{0.96}
\definecolor{1}{HTML}{F9F7ED}
\colorlet{2}{PaperAccent}

\hypersetup{
	colorlinks=true, % 将链接的边框（丑陋的方框）替换为彩色文本
	linkcolor=2,     % 内部链接的颜色（如公式引用、图表引用、Section跳转）
	citecolor=2,     % 文献引用的颜色（如 [1], [2]）
	urlcolor=black,  % 外部 URL 链接的颜色
	filecolor=black  % 本地文件链接的颜色
}

\providecommand{\AuthorInfoFontSize}{9.5}
\providecommand{\AuthorInfoLeading}{10.475}
\newsavebox{\AuthorInfoReference}
\newcommand{\AuthorInfoRows}{%
	$^{1}$The Hong Kong University of Science and Technology (Guangzhou)\\
	$^{2}$The Hong Kong University of Science and Technology\\
	$^{3}$Shanghai Xingyunzhili Artificial Intelligence Institute\\
	\texttt{zhengpengxie00@gmail.com} \qquad \texttt{li.zhang@xingyunzhili.com}\\
	\texttt{zekexie@hkust-gz.edu.cn} \qquad \texttt{maoyang@xingyunzhili.com}%
}
\AtBeginDocument{%
	\savebox{\AuthorInfoReference}{\normalfont\footnotesize
		\begin{tabular}[t]{@{}c@{}}\AuthorInfoRows\end{tabular}}%
}

\author{
	Zhengpeng Xie$^{1}$\thanks{Did this work during the internship at Xingyunzhili. $^{\dagger}$Project leader. $^{\ddagger}$Corresponding author.} \hspace{2em}
	Li Lyna Zhang$^{3,\dagger}$ \hspace{2em}
	Zeke Xie$^{1,2,\ddagger}$ \hspace{2em}
	Mao Yang$^{3}$\\
	\raisebox{0pt}[\ht\AuthorInfoReference][\dp\AuthorInfoReference]{%
		\normalfont\fontsize{\AuthorInfoFontSize pt}{\AuthorInfoLeading pt}\selectfont
		\begin{tabular}[t]{@{}c@{}}\AuthorInfoRows\end{tabular}}\\
	\raisebox{-14pt}[0pt][0pt]{\normalfont Our code is available at:\quad\hypersetup{urlcolor=PaperAccent}\url{https://github.com/xie-lab-ml/TOP-D}}
}

\begin{document}

% Keep v6's figure placement; use the unmodified template spacing above the names.
% No extra text line-spacing or paragraph-spacing adjustment is introduced.
\begingroup
\setlength{\intextsep}{6pt plus 1pt minus 1pt}
\maketitle
\vspace{-10pt}

\begin{figure*}[!h]
	\centering
	\includegraphics[scale=1.2]{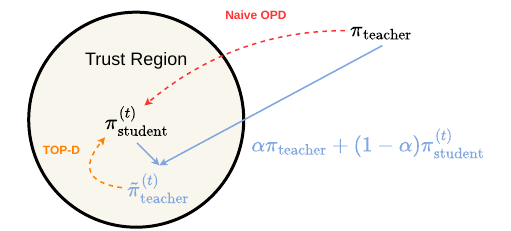}
	\caption{While standard OPD forces direct teacher supervision and suffers from unstable optimization, TOP-D ensures stable training by dynamically constructing a proximal teacher at every iteration.}
	\label{teaser}
\end{figure*}

\begin{abstract}
	Big goals are hard to achieve all at once; breaking them into small steps is wiser. We present Trust Region Policy Distillation (TOP-D), which transforms the notoriously unstable, high-variance On-Policy Distillation (OPD) into a stable training paradigm by dynamically constructing a \textit{proximal teacher}. Theoretically, we establish a rigorous framework demonstrating that TOP-D inherently controls gradient variance. By providing a formal global convergence analysis alongside a monotonic improvement bound, we mathematically formalize the reliability and stability of the overall training dynamics. Empirically, TOP-D dramatically enhances training stability, sample efficiency, and final performance on mathematical reasoning tasks. More importantly, TOP-D introduces zero additional computational overhead, positioning itself as a promising alternative to the well-established OPD paradigm.%
\end{abstract}

\begin{figure}[H]
	\centering
	\includegraphics[scale=0.5]{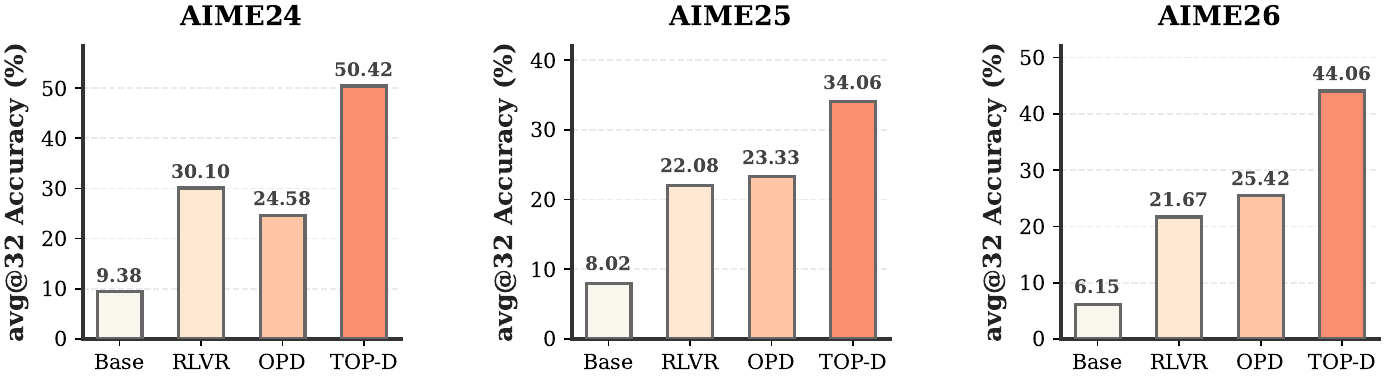}
	\caption{The avg@32 accuracy on several AIME benchmarks for RLVR applied to the Qwen3-8B-Base model, and for OPD and TOP-D using Qwen3-30B-A3B-Instruct-2507 as the teacher model.}
\end{figure}
\clearpage
\endgroup

\section{Introduction}
On-Policy Distillation (OPD) \citep{lu2025onpolicydistillation} has rapidly emerged as the prevailing paradigm for post-training Large Language Models (LLMs) \citep{yang2025qwen3, xiao2026mimo, zeng2026glm, yang2026nemotron, deepseekai2026deepseekv4}, in which the student learns from token-level signals provided by an external teacher. Mathematically, this is equivalent to reinforcement learning where the immediate reward is defined by the logarithmic probability ratio between the teacher and student policies. Consequently, the on-policy nature of OPD prevents the catastrophic forgetting of Supervised Fine-Tuning (SFT) \citep{shenfeld2026self}, while its dense rewards overcome the sample inefficiency and instability of Reinforcement Learning with Verifiable Rewards (RLVR) (as summarized in Table \ref{comparison of different paradigms}).

Despite its theoretical appeal, standard OPD suffers from severe optimization fragility in practice. The fundamental bottleneck of OPD lies in the inherent capacity gap between the target teacher and the student \citep{jang2026stable}. Because the standard distillation objective is driven by the logarithmic probability ratio between the two policies, the optimization dynamics are highly vulnerable to teacher-student disagreements. Recent studies have devoted considerable engineering efforts to mitigate this instability through various techniques, including teacher-student mixed sampling \citep{gu2024minillm}, reward clipping \citep{ko2026scaling}, top-$p$ sampling \citep{fu2026revisiting}, off-policy cold starts \citep{li2026rethinking}, and scaling to full-vocabulary supervision \citep{deepseekai2026deepseekv4}. Yet these mitigation strategies remain largely empirical heuristics that fail to capture the underlying mechanism of the instability: the unbounded logarithmic probability difference between the teacher and student. Furthermore, these empirical methods inherently lack theoretical guarantees for reliable optimization.

In this paper, we present Trust Region Policy Distillation (TOP-D) to systematically resolve the inherent optimization instability of standard OPD. Specifically, to overcome the unbounded logarithmic signals, we dynamically construct an \textit{external proximal teacher} (as illustrated in Figure \ref{teaser} and Section \ref{External Proximal Teacher}) by interpolating the probability distributions of the target teacher and the current student. This mathematical construction transforms the unbounded distillation reward into a smooth, strictly lower-bounded signal, analytically preventing variance explosion. Building upon this stabilized objective, we further incorporate \textit{internal trust region iterations} (described in Section \ref{Internal Trust Region Iterations}) to better approximate the current proximal teacher, while simultaneously improving overall sample efficiency.

Crucially, while the proximal teacher is constructed conceptually, the practical objective mathematically reduces to a simple \textit{algebraic transformation} of the standard token-level OPD reward. As a result, TOP-D introduces zero additional computational overhead, serving as a plug-and-play module that is natively compatible with both state-of-the-art distributed training infrastructures \citep{shoeybi2019megatron, zhao2023pytorch} and high-throughput inference engines \citep{kwon2023efficient, zheng2024sglang}. Beyond this empirical elegance, our proposed framework establishes a rigorous theoretical closed-loop: while the external proximal teacher strictly bounds the gradient variance (described in Section \ref{Bounded Variance}), the internal trust region iterations guarantee monotonic policy improvement (described in Section \ref{Monotonic Improvement}), systematically minimizing the single-step optimization error to tightly close the overall asymptotic convergence gap (described in Section \ref{Convergence Analysis}).

In summary, the main contributions of this paper are as follows:
\begin{itemize}
	\item We propose TOP-D, a plug-and-play distillation paradigm that smooths the unbounded reward signal and safely breaks the strict on-policy data-reuse barrier of standard OPD.
	\item We establish a rigorous theoretical framework where the proximal teacher inherently bounds gradient variance to stabilize training, while our convergence analysis and monotonic improvement bound form a closed-loop to systematically minimize the asymptotic error.
	\item We empirically demonstrate that TOP-D dramatically improves reasoning capabilities on mathematical reasoning tasks, delivering a 25.84\% absolute improvement in avg@32 accuracy over standard OPD on AIME24 and consistently dominating RLVR and OPD baselines.
\end{itemize}

\begin{table}[!t]
	\centering
	\renewcommand{\arraystretch}{1.2}
	\resizebox{\linewidth}{!}{
		\begin{tabular}{c c c c c}
			\toprule
			Method & Sampling & Reward Signal & Training Stability & Theoretical Guarantee \\
			\midrule
			
			Supervised Fine-Tuning 
			& off-policy & \textbf{dense} & \textbf{high} & weak \\
			
			\rowcolor{1}
			Reinforcement Learning with Verifiable Rewards
			& \textbf{on-policy} & sparse & low & weak \\
			
			On-Policy Distillation
			& \textbf{on-policy} & \textbf{dense} & low & weak \\
			
			\rowcolor{1}
			Trust Region Policy Distillation 
			& \textbf{on-policy} & \textbf{dense} & \textbf{high} & \textbf{strong} \\
			
			\bottomrule
		\end{tabular}
	}
	\caption{Comparison of different paradigms.}
	\label{comparison of different paradigms}
\end{table}

\section{Preliminaries}
\subsection{On-Policy Distillation}\label{On-Policy Distillation}
Suppose we have an autoregressive language model $\pi_{\theta}$, where $\theta\in\Theta$ denotes the parameters. Given a prompt $x\sim\mathcal{D}_x$, the model generates a response $y=(y_0,y_1,\dots,y_T)$\footnote{Here, $y_0$ denotes the special initial token and is not counted toward the response length.} autoregressively, where $T=\lvert y\rvert$ is the response length. Formally,
\begin{equation}
	\log\pi_{\theta}(y\mid x)=\sum_{t=1}^{\lvert y\rvert}\log\pi_{\theta}(y_t\mid x,y_{<t}),
\end{equation}
where $y_{<t}$ denotes $(y_0,y_1,\dots,y_{t-1})$. In the context of OPD, we aim to minimize the following \textit{reverse} KL divergence:
\begin{equation}\label{objective}
	\begin{split}
		\theta^*=\arg\min_{\theta}\mathcal{L}(\theta)&=\arg\min_{\theta}\mathbb{E}_{x\sim\mathcal{D}_x}\left[\mathcal{D}_{\mathrm{KL}}(\pi_{\theta}(\cdot\mid x),\pi^*(\cdot\mid x))\right]\\
		&=\arg\max_{\theta}\mathbb{E}_{x\sim\mathcal{D}_x,y\sim\pi_{\theta}(\cdot\mid x)}\left[\log\frac{\pi^*(y\mid x)}{\pi_{\theta}(y\mid x)}\right]=\arg\max_{\theta}J(\theta),\\
	\end{split}
\end{equation}
where $\pi^*$ is typically a stronger model. We are now ready to present the standard \textit{policy gradient} form of the objective \eqref{objective}:
\begin{equation}\label{policy gradient}
	\nabla_{\theta}J(\theta)=\mathbb{E}_{x\sim\mathcal{D}_x,y\sim\pi_{\theta}(\cdot\mid x)}\left[\sum_{t=1}^{\lvert y\rvert}\nabla_{\theta}\log\pi_{\theta}(y_t\mid x,y_{<t})\cdot\left(\sum_{k=t}^{\lvert y\rvert}\log\rho_k\right)\right],
\end{equation}
where $\rho_k=\frac{\pi^*(y_k\mid x,y_{<k})}{\pi_{\theta}(y_k\mid x,y_{<k})}$. In other words, we can treat $\log\rho_k$ as the immediate reward for the $k$-th token, i.e., $r_k=\log\rho_k$. However, the cumulative reward $\sum_{k=t}^{\lvert y\rvert}r_k$ introduces an inherent length bias as well as high variance. In practice, a common approach is to use the immediate reward $r_k$ directly.

\subsection{Reinforcement Learning Formulation}\label{Reinforcement Learning Formulation}
To establish a rigorous theoretical foundation, we cast autoregressive language modeling as a deterministic Markov Decision Process (MDP) defined by the tuple $\left(\mathcal{S},\mathcal{A},\mathcal{P},r,\mu_0\right)$.
\begin{itemize}
	\item \textbf{State space $\mathcal{S}$ and action space $\mathcal{A}$.} State $s_t = (x, y_{<t})$ represents the prompt and prefix. The action space $\mathcal{A} = \mathcal{V} $ includes the vocabulary and the \texttt{EOS} token. To rigorously handle variable-length sequences, we augment $\mathcal{S}$ with a permanent \textit{absorbing state} $s_\bot$.
	\item \textbf{Dynamics $\mathcal{P}$ and initial distribution $\mu_0$.} The initial state distribution $\mu_0(s)$ follows the prompt distribution $\mathcal{D}_x$ and the special initial token $y_0$. Transitions are \textit{deterministic}: selecting $a \neq \texttt{EOS}$ appends $a$ to the prefix, while selecting $\texttt{EOS}$ definitively transitions to $s_\bot$.
	\item \textbf{Reward function $r$.} The immediate reward is the distillation signal: $r(s, a) = \log \frac{\pi^*(a \mid s)}{\pi(a \mid s)}$. To prevent infinite accumulation after termination\footnote{Assume that $\pi(\texttt{EOS}\mid s_\bot)=1$.}, we set $r(s_\bot, a) \equiv 0$ for any $a\in\mathcal{A}$.
\end{itemize}

Let $\pi: \mathcal{S} \to \Delta(\mathcal{A})$ be a stochastic policy. In the undiscounted setting ($\gamma=1$), the action value function $Q^\pi(s, a)$ is defined as $Q^\pi(s, a) = \mathbb{E}_{\pi} \left[ \sum_{k=0}^{\infty} r(s_{t+k}, a_{t+k})\mid s_t=s, a_t=a \right]$. The state value function $V^\pi(s)$ is strictly obtained by marginalizing $Q^\pi$ over the policy: $V^\pi(s) = \mathbb{E}_{a \sim \pi(\cdot \mid s)} \left[ Q^\pi(s, a) \right]$. Consequently, we define the advantage function, which forms the basis for policy improvement bounds, as $A^\pi(s, a) = Q^\pi(s, a) - V^\pi(s)$. To decouple sequential dependencies, we introduce the \textit{unnormalized} state distribution $d^\pi(s) = \sum_{t=1}^{\infty}\mathbb{P}(s_t = s \mid \pi)$. The true objective of a policy $\pi$ is the expected cumulative reward over the infinite horizon:
\begin{equation}\label{original eta}
	\eta(\pi) = \mathbb{E}_{x \sim \mathcal{D}_x, y \sim \pi(\cdot \mid x)} \left[ \sum_{t=1}^{\infty} r(s_t, a_t) \right].
\end{equation}

While this abstract MDP formalization is adopted for the theoretical proofs in Section \ref{Monotonic Improvement}, the remainder of the paper defaults to the notation established in Section \ref{On-Policy Distillation}.

\section{Methodology}
This section details the design of our core algorithm, Trust Region Policy Distillation (TOP-D). We begin in Section \ref{External Proximal Teacher} by constructing a \textit{proximal teacher} to stabilize training. Subsequently, in Section \ref{Internal Trust Region Iterations}, the current student approximates this proximal teacher by performing \textit{trust region iterations}. Finally, Section \ref{Trust Region Policy Distillation} outlines the practical implementation of the TOP-D algorithm.

\subsection{External Proximal Teacher}\label{External Proximal Teacher}
The fundamental challenge in standard On-Policy Distillation (OPD) arises from the inherent capacity gap between the student policy $\pi_{\theta}$ and the target teacher policy $\pi^*$. As defined in Equation \eqref{policy gradient}, the optimization dynamics heavily rely on the probability ratio $\rho_k=\frac{\pi^*(y_k\mid x,y_{<k})}{\pi_{\theta}(y_k\mid x,y_{<k})}$. In extreme cases where the teacher assigns a near-zero probability to a generated token ($\pi^*\rightarrow0$), the immediate reward $r_k=\log\rho_k$ diverges to negative infinity. This unbounded penalty destabilizes the policy gradient, making the optimization process extremely fragile.

\begin{wrapfigure}{r}{0.38\textwidth}
	\centering
	\includegraphics[width=0.35\textwidth]{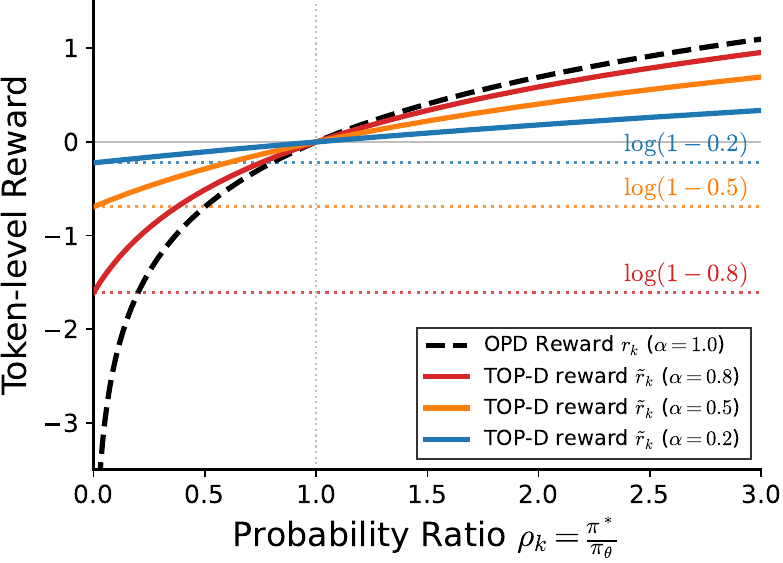}
	\caption{Reward curves of $r_k$ and $\tilde{r}_k$.}
	\label{reward comparison}
\end{wrapfigure}

To address this instability, we draw inspiration from \textit{trust region} methods in reinforcement learning \citep{schulman2015trust, schulman2017proximal, xie2025simple}, which restrict policy updates from taking destructively large steps. Rather than forcing the student to perfectly imitate a rigid teacher, we dynamically construct an intermediate, localized target distribution. We term this surrogate target the \textit{proximal teacher} $\tilde{\pi}^*$. Formally,
\begin{equation*}
	\tilde{\pi}^*(y_k\mid x,y_{<k})=\alpha\pi^*(y_k\mid x,y_{<k})+(1-\alpha)\pi_{\theta}(y_k\mid x,y_{<k}),
\end{equation*}
where $\alpha\in(0,1)$ is the coefficient that controls the interpolation strength. In fact, the proximal teacher inherently smooths the token-level reward $r_k$. Specifically,
\begin{equation}
	\tilde{r}_k=\log\frac{\tilde{\pi}^*(y_k\mid x,y_{<k})}{\pi_{\theta}(y_k\mid x,y_{<k})}=\log\left(\alpha\frac{\pi^*(y_k\mid x,y_{<k})}{\pi_{\theta}(y_k\mid x,y_{<k})}+1-\alpha\right)=\log\left(\alpha\rho_k+1-\alpha\right).
\end{equation}
When $\alpha=1$, it reduces to the original OPD reward $r_k$. Thanks to mathematics, we do not need to explicitly construct a proximal teacher to obtain the new reward $\tilde{r}_k$. Figure \ref{reward comparison} visualizes the OPD reward $r_k$ and our proposed reward $\tilde{r}_k$ with respect to the probability ratio $\rho_k$. We can see that $\tilde{r}_k$ is strictly bounded from below\footnote{$\log(\alpha x+1-\alpha)\geq\log(1-\alpha)$ when $x\in(0,+\infty)$ and $\alpha\in(0,1)$.} as $\rho_k\rightarrow0$. Moreover, we explicitly justify our choice of interpolating in the \textit{probability space} instead of the \textit{log-probability space} to construct the proximal teacher, since setting $\log\tilde{\pi}^*(y_k\mid x,y_{<k})=\alpha\log\pi^*(y_k\mid x,y_{<k})+(1-\alpha)\log\pi_{\theta}(y_k\mid x,y_{<k})$ yields
\begin{equation}
	\tilde{r}_k=\alpha r_k,
\end{equation}
which is merely a naive scaling of the original OPD reward.

\subsection{Internal Trust Region Iterations}\label{Internal Trust Region Iterations}
With the proximal teacher established, the student policy is optimized to maximize the expected cumulative reward $\tilde{R}_t=\sum_{k=t}^{\lvert y\rvert}\tilde{r}_k$. While standard OPD relies on continuous online sampling, it fundamentally operates in a strictly on-policy manner. In the naive formulation, once the student policy is updated, the previously generated trajectories are immediately discarded. This results in notoriously low sample efficiency, as OPD lacks the capacity to reuse data.

Inspired by modern trust region algorithms \citep{schulman2015trust, schulman2017proximal, xie2025simple}, we analogously decouple the \textit{behavior policy} $\pi_{\theta_{\mathrm{old}}}$ and the \textit{target policy} $\pi_{\theta}$ to inherit their high sample efficiency. Specifically,
\begin{equation}\label{ppo objective}
	\mathcal{J}(\theta)=\mathop{\mathbb{E}}_{\substack{x\sim\mathcal{D}_x\\\{y^{i}\}_{i=1}^{G}\sim\pi_{\theta_{\mathrm{old}}}(\cdot\mid x)}}\left[\frac{1}{\sum_{i=1}^{G}\lvert y^{i}\rvert}\sum_{i=1}^{G}\sum_{t=1}^{\lvert y^{i}\rvert}\left\{\min\left[p_t^i\cdot\hat{A}_t^i,\mathrm{clip}\left(p_t^i,1-\epsilon,1+\epsilon\right)\cdot\hat{A}_t^i\right]\right\}\right].
\end{equation}
Here, we follow the group-based reinforcement learning settings \citep{shao2024deepseekmath, guo2025deepseek}, where the superscript $i$ indexes different model responses and $p_t^i=\frac{\pi_{\theta}(y_t^i\mid x,y_{<t}^i)}{\pi_{\theta_{\mathrm{old}}}(y_t^i\mid x,y_{<t}^i)}$ is the importance sampling ratio. We adopt a more fine-grained token-level normalized advantage $\hat{A}_t^i$ to better exploit dense reward signals, as illustrated in Figure \ref{advantage}. For each prompt, the advantage relies on a token-level return $\tilde{R}_k^i$, which combines the immediate reward and a length-normalized future return:
\begin{equation}
	\tilde{R}_k^i=\tilde{r}_k^i+\frac{\sum_{j=k+1}^{\lvert y^i\rvert}\tilde{r}_j^i}{\max\{\lvert y^i\rvert-k,1\}}
	=\begin{cases}
		\tilde{r}_k^i+\frac{1}{\lvert y^i\rvert-k}\sum_{j=k+1}^{\lvert y^i\rvert}\tilde{r}_j^i, & k<\lvert y^i\rvert,\\
		\tilde{r}_k^i, & k=\lvert y^i\rvert.
	\end{cases}
\end{equation}
Empty sums are zero. Empirically, this return discourages both overly short and excessively long responses. For any given prompt, the mean $\mu$ and standard deviation $\sigma$ within each group are calculated via $\mu=\mathrm{mean}\{\tilde{R}_k^i\}$ and $\sigma=\mathrm{std}\{\tilde{R}_k^i\}$, where $i=1,\dots,G$ and $k=1,\dots,\lvert y^i\rvert$.

\begin{figure*}[!t]
	\centering
	\includegraphics[scale=0.7]{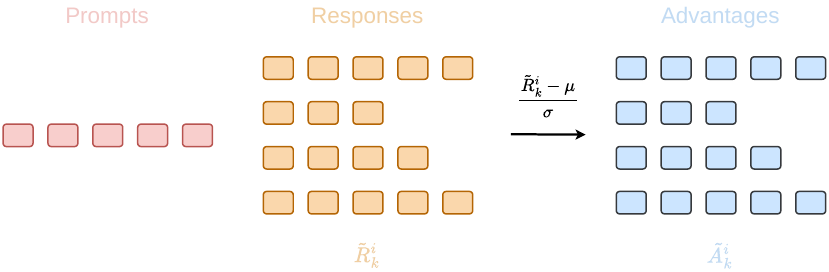}
	\caption{In contrast to \textit{sequence-level} advantage normalization, we perform \textit{token-level} normalization across the responses generated for a given prompt to better leverage the dense reward signal.}
	\label{advantage}
\end{figure*}

\subsection{Trust Region Policy Distillation}\label{Trust Region Policy Distillation}
In this section, we synthesize the concepts of the \textit{external proximal teacher} in Section \ref{External Proximal Teacher} and the \textit{internal trust region iterations} in Section \ref{Internal Trust Region Iterations} to formulate our complete algorithmic framework, Trust Region Policy Distillation (TOP-D), summarized in Algorithm \ref{topd}.

Standard OPD suffers from two fundamental bottlenecks: optimization fragility and low sample efficiency. TOP-D systematically resolves both through a unified design. First, the \textit{external proximal teacher} transforms the unbounded OPD reward into a smooth, lower-bounded signal $\tilde{r}_k$, inherently stabilizing the optimization process. Building upon this principle, the \textit{internal trust region iterations} decouple the behavior and target policies, breaking the strict on-policy data-reuse barrier. Crucially, these mechanisms are not merely empirical heuristics but share deep theoretical synergy. We rigorously formalize this framework and establish its mathematical guarantees in Section \ref{Theoretical Analysis}.

\begin{algorithm}[!t]
	\caption{Trust Region Policy Distillation (TOP-D)}
	\label{topd}
	\begin{algorithmic}[1]
		\REQUIRE Student $\pi_{\theta}$, teacher $\pi^*$, dataset $\mathcal{D}_x$, interpolation $\alpha$, group size $G$, internal epochs $E$
		\WHILE{not converged}
		\STATE Initialize $\pi_{\theta_{\mathrm{old}}} \leftarrow \pi_{\theta}$
		\STATE Sample a batch of prompts $\mathcal{X}\subset\mathcal{D}_x$
		\FOR{each prompt $x\in\mathcal{X}$}
		\STATE Perform rollout to sample $G$ responses $\{y^{i}\}_{i=1}^{G}\sim\pi_{\theta_{\mathrm{old}}}(\cdot\mid x)$
		\STATE Compute TOP-D rewards $\tilde{r}_k^i\leftarrow\log\left(\alpha\frac{\pi^*(y_k^i\mid x,y_{<k}^i)}{\pi_{\theta_{\mathrm{old}}}(y_k^i\mid x,y_{<k}^i)}+1-\alpha\right)$
		\STATE Compute token-level returns $\tilde{R}_k^i\leftarrow\tilde{r}_k^i+\frac{1}{\max\{\lvert y^i\rvert-k,1\}}\sum_{j=k+1}^{\lvert y^i\rvert}\tilde{r}_j^i$
		\STATE Compute token-level advantages $\hat{A}_k^i\leftarrow\frac{\tilde{R}_k^i-\mu}{\sigma}$
		\ENDFOR
		\FOR{epoch $e = 1$ to $E$}
		\STATE $\mathcal{J}(\theta)\leftarrow\frac{1}{\sum_{i=1}^{G}\lvert y^{i}\rvert}\sum_{i=1}^{G}\sum_{t=1}^{\lvert y^{i}\rvert}\left\{\min\left[p_t^i\cdot\hat{A}_t^i,\mathrm{clip}\left(p_t^i,1-\epsilon,1+\epsilon\right)\cdot\hat{A}_t^i\right]\right\}$
		\STATE Update $\theta$ by maximizing $\mathcal{J}(\theta)$ using an optimizer
		\ENDFOR
		\ENDWHILE
	\end{algorithmic}
\end{algorithm}

\section{Theoretical Analysis}\label{Theoretical Analysis}
In this section, we establish a rigorous theoretical foundation for TOP-D. We begin in Section \ref{Bounded Variance} by proving that the proximal teacher inherently bounds gradient variance to stabilize optimization. Building upon this, we derive a global convergence bound in Section \ref{Convergence Analysis} by abstracting the learning dynamics. Finally, to minimize the local error bottleneck identified in this convergence analysis, we establish the monotonic improvement guarantee for our internal trust region iterations in Section \ref{Monotonic Improvement}.

\subsection{Bounded Variance}\label{Bounded Variance}
To rigorously analyze the gradient stability, we restrict our discussion to the autoregressive generation environment over a finite vocabulary $\mathcal{V}$. In policy gradient methods, the variance of the gradient is typically dominated by its \textit{second moment}, i.e., $\operatorname{tr}(\mathrm{Var}(g))\leq\mathbb{E}[\lVert g\rVert^2]$. Thus, our theoretical analysis will primarily focus on establishing an absolute upper bound for the gradient's second moment.
\begin{assumption}\label{assumption 1}
	We assume that the gradient of the student policy's log-probability (the score function) is uniformly bounded. That is, there exists a constant $M>0$, such that for any prompt $x\sim\mathcal{D}_x$, generated response $y\sim\pi(\cdot\mid x)$, and $k$-th token $y_k$, we have
	\begin{equation}
		\lVert\nabla_{\theta}\log\pi_{\theta}(y_k\mid x,y_{<k})\rVert\leq M.
	\end{equation}
\end{assumption}
Let $g_k=\nabla_{\theta}\log\pi_{\theta}(y_k\mid x,y_{<k})\cdot r_k$ denote the token-level policy gradient estimator for standard OPD, and $\tilde{g}_k=\nabla_{\theta}\log\pi_{\theta}(y_k\mid x,y_{<k})\cdot\tilde{r}_k$ denote the gradient estimator for our proposed TOP-D.
\begin{theorem}\label{bounded variance}
	Under Assumption \ref{assumption 1}, for any interpolation coefficient $\alpha\in(0,1)$, the variance of the TOP-D gradient estimator $\tilde{g}_k$ does not diverge across the entire probability space. Instead, it is strictly bounded by a uniform absolute limit:
	\begin{equation}
		\operatorname{tr}(\mathrm{Var}(\tilde{g}_k))\leq M^2\left(\log^2(1-\alpha)+C^*\alpha\right),
	\end{equation}
	where $C^*$ is a universal mathematical constant independent of the vocabulary size.
\end{theorem}
\begin{proof}
	See Appendix \ref{proof of bounded variance}.
\end{proof}
\begin{remark}
	TOP-D employs an asymmetric operator to stabilize training. For heavy penalizations, it introduces a hard safety bound, $\log^2(1-\alpha)$, preventing catastrophic divergence; for high rewards, it applies a linear dampening scaled by $C^*\alpha$. This effectively designates $\alpha$ as a direct variance controller. As $\alpha \to 1^-$, TOP-D recovers standard OPD, with its upper bound diverging ($\log^2(1-\alpha) \to \infty$) to reflect the lack of explicit variance control in OPD. Conversely, as $\alpha \to 0^+$, the variance vanishes.
\end{remark}

\subsection{Convergence Analysis}\label{Convergence Analysis}
To rigorously analyze the convergence properties of TOP-D, we abstract the construction of the proximal teacher as an iterative operator mapping over the policy space. Let $\Pi$ denote the space of all valid policies. To quantify the divergence within this space, we introduce the expected $L_1$ distance between any two policies $\pi_i, \pi_j$ as $d(\pi_i, \pi_j) = \mathbb{E}_{x \sim \mathcal{D}_x} \left[ \lVert \pi_i(\cdot \mid x) - \pi_j(\cdot \mid x) \rVert_1 \right]$. Next, we formalize the iterative process through the \textit{proximal teacher operator} $\mathcal{T}:\Pi\to\Pi$, given by
\begin{equation}\label{proximal teacher operator}
	\mathcal{T}(\pi)=\alpha\pi^*+(1-\alpha)\pi,
\end{equation}
where $\alpha\in(0,1)$ is the interpolation coefficient, and $\pi^*$ is the target teacher. Building upon this metric space, we are now ready to establish the following convergence analysis:
\begin{theorem}\label{convergence}
	For any initial student policy $\pi_0\in\Pi$, consider the iterative update process $\pi_{k+1}=\mathcal{T}(\pi_k)+\epsilon_k$\footnote{With a slight abuse of notation, let $\lVert\epsilon_k\rVert_1=d(\pi_{k+1},\tilde{\pi}_k^*)$, where $\tilde{\pi}_k^*=\mathcal{T}(\pi_k)$ is the proximal teacher at the $k$-th step.}, where $\epsilon_k$ accounts for the practical optimization error (e.g., function approximation and optimization noise) at the $k$-th global step, where $k\geq0$. Then, the following bound holds:
	\begin{equation}
		d(\pi_{k+1},\pi^*)\leq(1-\alpha)^{k+1}d(\pi_0,\pi^*)+\sum_{i=0}^{k}(1-\alpha)^{k-i}\lVert\epsilon_i\rVert_1.
	\end{equation}
	Furthermore, assuming the optimization process eventually stabilizes such that the asymptotic error is bounded by $\limsup_{k\rightarrow \infty}\lVert\epsilon_{k}\rVert_1\leq \epsilon_{\infty}$, we have
	\begin{equation}
		\limsup_{k \rightarrow \infty} d(\pi_{k+1}, \pi^*) \leq \frac{\epsilon_{\infty}}{\alpha}.
	\end{equation}
\end{theorem}
\begin{proof}
	See Appendix \ref{proof of convergence}.
\end{proof}
\begin{remark}
	The bound presented in Theorem \ref{convergence} not only reveals an inherent \textit{error forgetting} mechanism that isolates early updates, but also highlights a fundamental \textit{theoretical bottleneck}: the asymptotic convergence gap is bounded by $\epsilon_\infty/\alpha$. This implies that to break through the precision ceiling without compromising stability, i.e., keeping $\alpha$ small (we offer more details in Appendix \ref{Asymptotic Error}), the only viable path is to strictly minimize the single-step optimization error $\lVert\epsilon_k\rVert_1$. This theoretical bottleneck motivates us to introduce \textit{internal trust region iterations} in the subsequent section.
\end{remark}

\subsection{Monotonic Improvement}\label{Monotonic Improvement}
As revealed in Theorem \ref{convergence}, the fundamental theoretical bottleneck of the proximal teacher framework lies in the asymptotic convergence gap bounded by $\epsilon_\infty/\alpha$. To tightly close this gap, it is imperative to rigorously minimize the single-step optimization error $\lVert\epsilon_k\rVert_1$. TOP-D achieves this via internal trust region iterations. In this section, we establish the theoretical insight of this mechanism by proving monotonic policy improvement. For brevity, we follow the same notation established in Section \ref{Reinforcement Learning Formulation}.
\begin{assumption}\label{assumption 2}
	Given the fixed context window constraints inherent to autoregressive language models, we assume that any valid policy $\pi \in \Pi$ generates the \texttt{EOS} token within a finite maximum horizon $T_{\max}$. Consequently, the response length strictly satisfies $\lvert y\rvert \leq T_{\max}$, so $\sum_{s\neq s_\bot}d^{\pi}(s)<\infty$.
\end{assumption}
In our undiscounted setting ($\gamma=1$), the unnormalized measure $d^\pi(s)$ diverges because the permanent absorbing state $s_\bot$ dominates as $t \to \infty$. To construct a valid probability distribution over the meaningful transient states, we explicitly exclude $s_\bot$ and normalize $d^\pi(s)$.
\begin{definition}\label{normalized state distribution}
	We formally define the \textit{expected response length} of a policy $\pi$ as $\ell_\pi = \mathbb{E}_{x \sim \mathcal{D}_x, y \sim \pi(\cdot \mid x)}\left[\lvert y\rvert\right]$. Excluding $s_\bot$, the sum of the state distribution precisely evaluates to this length. Thus, we define the normalized state visitation measure $d^\pi_{\mathrm{norm}}(s) = \frac{d^\pi(s)}{\ell_\pi}$ for $s \neq s_\bot$.
\end{definition}
\begin{remark}
	This normalized measure enables an equivalent reformulation of the \textit{Performance Difference Lemma} \citep{kakade2002approximately}, strictly decoupling the sequence length from the single-step generation quality.
\end{remark}
We can thus rewrite the original objective \eqref{original eta} as $\eta(\pi)=\ell_{\pi}\cdot\mathbb{E}_{s \sim d^{\pi}_{\mathrm{norm}}(\cdot), a \sim\pi(\cdot \mid s)} \left[r(s,a)\right]$, and we are now ready to present the main theoretical result of this section:
\begin{theorem}\label{performance lower bound}
	Given Assumption \ref{assumption 2} and Definition \ref{normalized state distribution}, for any two policies, $\tilde{\pi}$ and $\pi$, the following bound holds:
	\begin{equation}
		\eta(\tilde{\pi})\geq\zeta_\pi(\tilde{\pi}) - 2\xi T_{\max}\ell_\pi\mathbb{E}_{s \sim d^{\pi}_{\mathrm{norm}}(\cdot)}\left[\mathcal{D}_{\mathrm{TV}}(\tilde{\pi}(\cdot \mid s),\pi(\cdot \mid s))\right],
	\end{equation}
	where $\zeta_\pi(\tilde{\pi})=\eta(\pi) + \ell_\pi\mathbb{E}_{s \sim d^\pi_{\mathrm{norm}}(\cdot), a \sim \tilde{\pi}(\cdot \mid s)} \left[ A^\pi(s, a) \right]$, and $\xi=\max_{s}\lvert\mathbb{E}_{a\sim\tilde{\pi}(\cdot\mid s)}\left[A^\pi(s,a)\right]\rvert$.
\end{theorem}
\begin{proof}
	See Appendix \ref{proof of performance lower bound}.
\end{proof}
\begin{remark}
	Let $\mathcal{M}_{\pi}(\tilde{\pi})$ denote the lower bound established in Theorem \ref{performance lower bound}. By definition, $\eta(\pi)=\mathcal{M}_{\pi}(\pi)$. Therefore, any policy update that improves the lower bound guarantees monotonic improvement on the true objective, i.e., $\eta(\tilde{\pi}) - \eta(\pi) \geq \mathcal{M}_\pi(\tilde{\pi}) - \mathcal{M}_\pi(\pi)$.
\end{remark}
As a result, optimizing $\mathcal{M}_{\pi}(\tilde{\pi})$ ensures that the original reverse KL divergence objective monotonically improves during the internal updates. Let the superscript $(n)$ denote the $n$-th step of the internal trust region iterations at the $k$-th global step, and let $\tilde{\pi}_k^* = \mathcal{T}(\pi_k)$ be the fixed proximal teacher. By iteratively optimizing $\mathcal{M}_{\pi}(\tilde{\pi})$, we obtain a monotonically decreasing sequence of KL divergences:
\begin{equation}\label{decreasing sequence of KL divergences}
	\mathbb{E}_{x\sim\mathcal{D}_x}\left[\mathcal{D}_{\mathrm{KL}}^{(0)}\right]\geq\mathbb{E}_{x\sim\mathcal{D}_x}\left[\mathcal{D}_{\mathrm{KL}}^{(1)}\right]\geq\cdots\geq\mathbb{E}_{x\sim\mathcal{D}_x}\left[\mathcal{D}_{\mathrm{KL}}^{(n)}\right],
\end{equation}
where $\mathcal{D}_{\mathrm{KL}}^{(j)} \triangleq \mathcal{D}_{\mathrm{KL}}(\pi_k^{(j)}(\cdot \mid x),\tilde{\pi}_k^*(\cdot \mid x))$ denotes the reverse KL divergence objective at the $j$-th internal iteration. If the internal trust region optimization proceeds until the final expected divergence satisfies $\mathbb{E}_{x \sim \mathcal{D}_x}\left[\mathcal{D}_{\mathrm{KL}}^{(n)}\right] \leq \frac{1}{2}\delta^2$ for $\delta>0$, then by combining Pinsker's inequality with Jensen's inequality: $\lVert\epsilon_{k}\rVert_1=d(\pi_{k+1},\tilde{\pi}_k^*)=d(\pi_k^{(n)}, \tilde{\pi}_k^*)\leq\mathbb{E}_{x \sim \mathcal{D}_x}\left[\sqrt{2\mathcal{D}_{\mathrm{KL}}^{(n)}}\right]\leq\sqrt{2\mathbb{E}_{x\sim\mathcal{D}_x}\left[\mathcal{D}_{\mathrm{KL}}^{(n)}\right]}\leq\delta$, we rigorously guarantee that the single-step optimization error satisfies $\lVert\epsilon_k\rVert_1\leq\delta$. This result perfectly closes the overall theoretical loop: the internal trust region iterations mathematically guarantee the continuous minimization of the single-step optimization error $\epsilon_k$, thereby enabling our proposed algorithm to tightly close the asymptotic convergence gap $\epsilon_\infty/\alpha$ identified in Theorem \ref{convergence}.

\section{Experiments}
In this section, we empirically validate the effectiveness of TOP-D. Specifically, Section \ref{Experimental Setup} details our validation datasets, model configurations, baseline methods, and implementation hyperparameters. Section \ref{Main Results} presents a comprehensive comparison of TOP-D against existing post-training baselines. Finally, Section \ref{Ablation Study} provides an in-depth analysis of the core components of our method, including the external proximal teacher and the internal trust region iterations, as well as the sensitivity analysis.

\subsection{Experimental Setup}\label{Experimental Setup}
\textbf{Datasets and benchmarks.} We conduct our training phase using the DAPO-Math-17k \citep{yu2025dapo} dataset, a high-quality mathematical reasoning corpus. For validation, we benchmark the post-trained models on widely recognized and challenging mathematical reasoning competitions like AIME and AMC.

\textbf{Models and baselines.} To comprehensively investigate the scalability of TOP-D and its robustness across different teacher-student capacity gaps, we employ two distinct base models as students: Qwen3-1.7B-Base \citep{yang2025qwen3} and Qwen3-8B-Base \citep{yang2025qwen3}. For the target teacher policies, we utilize Qwen3-14B \citep{yang2025qwen3} and the exceptionally capable Qwen3-30B-A3B-Instruct-2507 \citep{yang2025qwen3}. We compare TOP-D against several post-training paradigms including GRPO \citep{shao2024deepseekmath}, DAPO \citep{yu2025dapo}, and standard OPD \citep{lu2025onpolicydistillation}.

\textbf{Implementation details.} During the rollout phase, the student model generates 8 responses for each prompt. To execute the trust region iterations, we follow the same off-policy setting as RLVR by dividing the same batch of data into multiple mini-batches, where the global batch size is set to 512 prompts (4096 samples) and the mini-batch size is set to 32 prompts (256 samples). The core interpolation coefficient $\alpha$ for TOP-D is set to 0.1 or 0.2. During training, we apply a sampling temperature of 1.0 and a top-$p$ of 1.0. For performance validation, we use a temperature of 1.0 and a top-$p$ of 0.7. Comprehensive details regarding other hyperparameters are provided in Appendix \ref{Hyperparameters}.

\begin{table}[!h]
	\centering
	\renewcommand{\arraystretch}{1.1}
	\resizebox{\linewidth}{!}{
		\begin{tabular}{llc ccc ccc}
			\toprule
			\multirow{4}{*}[-0.8em]{\textbf{Student Model}} & \multirow{4}{*}[-0.8em]{\textbf{Paradigm}} & \multirow{4}{*}[-0.8em]{\textbf{Method}} & \multicolumn{6}{c}{\textbf{Teacher Model}} \\
			\cmidrule(lr){4-9}
			& & & \multicolumn{6}{c}{\textit{\textbf{Qwen3-30B-A3B-Instruct-2507}}} \\
			\cmidrule(lr){4-9}
			& & & \multicolumn{4}{c}{\textbf{avg@32}} & \multicolumn{2}{c}{\textbf{avg@8}} \\
			\cmidrule(lr){4-7} \cmidrule(lr){8-9}
			& & & \textbf{AIME24} & \textbf{AIME25} & \textbf{AIME26} & \textbf{AMC23} & \textbf{MATH-500} & \textbf{Olympiad} \\
			\midrule
			
			\multirow{5}{*}{\textit{\textbf{Qwen3-8B-Base}}} 
			& \textbf{Base} & - & 9.38 & 8.02 & 6.15 & 53.13 & 75.23 & 40.08 \\
			\cmidrule(l){2-9}
			& \multirow{2}{*}{\textbf{RLVR}} & GRPO & 30.10 & 22.08 & 21.67 & 56.33 & 76.83 & 44.92 \\
			& & DAPO & 32.92 & 27.81 & 32.29 & 65.39 & 81.65 & 44.23 \\
			\cmidrule(l){2-9}
			& \multirow{2}{*}{\textbf{Distillation}} & OPD & 24.58 & 23.33 & 25.42 & 76.88 & 87.98 & 59.29 \\
			& & \cellcolor{1}\textbf{TOP-D (Ours)} & \cellcolor{1}\textbf{50.42} {\color{2}\scriptsize\textbf{(+25.84)}} & \cellcolor{1}\textbf{34.06} {\color{2}\scriptsize\textbf{(+10.73)}} & \cellcolor{1}\textbf{44.06} {\color{2}\scriptsize\textbf{(+18.64)}} & \cellcolor{1}\textbf{88.13} {\color{2}\scriptsize\textbf{(+11.25)}} & \cellcolor{1}\textbf{91.23} {\color{2}\scriptsize\textbf{(+3.25)}} & \cellcolor{1}\textbf{64.67} {\color{2}\scriptsize\textbf{(+5.38)}} \\
			\bottomrule
		\end{tabular}
	}
	\caption{Performance comparison of the Qwen3-8B-Base across various mathematical benchmarks.}
	\label{8b}
\end{table}

\begin{table}[!h]
	\centering
	\renewcommand{\arraystretch}{1.1}
	\resizebox{\linewidth}{!}{
		\begin{tabular}{llc ccc ccc}
			\toprule
			\multirow{4}{*}[-0.8em]{\textbf{Student Model}} & \multirow{4}{*}[-0.8em]{\textbf{Paradigm}} & \multirow{4}{*}[-0.8em]{\textbf{Method}} & \multicolumn{6}{c}{\textbf{Teacher Model}} \\
			\cmidrule(lr){4-9}
			& & & \multicolumn{3}{c}{\textit{\textbf{Qwen3-30B-A3B-Instruct-2507}}} & \multicolumn{3}{c}{\textit{\textbf{Qwen3-14B}}} \\
			\cmidrule(lr){4-6} \cmidrule(lr){7-9}
			& & & \multicolumn{6}{c}{\textbf{avg@32}} \\
			\cmidrule(lr){4-9}
			& & & \textbf{AIME24} & \textbf{AIME25} & \textbf{AIME26} & \textbf{AIME24} & \textbf{AIME25} & \textbf{AIME26} \\
			\midrule
			
			\multirow{5}{*}{\textit{\textbf{Qwen3-1.7B-Base}}} 
			& \textbf{Base} & - & 3.44 & 1.98 & 1.88 & 3.44 & 1.98 & 1.88 \\
			\cmidrule(l){2-9}
			& \multirow{2}{*}{\textbf{RLVR}} & GRPO & 10.52 & 7.19 & 4.69 & 10.52 & 7.19 & 4.69 \\
			& & DAPO & 12.29 & 11.56 & 8.54 & 12.29 & 11.56 & 8.54 \\
			\cmidrule(l){2-9}
			& \multirow{2}{*}{\textbf{Distillation}} & OPD & 8.96 & 7.50 & 5.94 & 7.81 & 8.33 & 5.63 \\
			& & \cellcolor{1}\textbf{TOP-D (Ours)} & \cellcolor{1}\textbf{20.31} {\color{2}\scriptsize\textbf{(+11.35)}} & \cellcolor{1}\textbf{17.71} {\color{2}\scriptsize\textbf{(+10.21)}} & \cellcolor{1}\textbf{13.75} {\color{2}\scriptsize\textbf{(+7.81)}} & \cellcolor{1}\textbf{12.81} {\color{2}\scriptsize\textbf{(+5.00)}} & \cellcolor{1}\textbf{14.69} {\color{2}\scriptsize\textbf{(+6.36)}} & \cellcolor{1}\textbf{9.69} {\color{2}\scriptsize\textbf{(+4.06)}} \\
			\bottomrule
		\end{tabular}
	}
	\caption{Performance comparison of the Qwen3-1.7B-Base across AIME benchmarks.}
	\label{1.7b}
\end{table}

\subsection{Main Results}\label{Main Results}
Table \ref{8b} and Table \ref{1.7b} summarize the validation results of TOP-D compared to other baselines. TOP-D consistently outperforms the evaluated baselines across all mathematical benchmarks and student model scales. The most striking improvement is observed when compared to standard OPD. For instance, using the Qwen3-8B-Base student supervised by the Qwen3-30B-A3B-Instruct-2507 teacher, TOP-D achieves an avg@32 accuracy of 50.42\% on AIME24, delivering a massive 25.84\% absolute improvement over standard OPD (24.58\%). This immense performance gain definitively proves that bounding the variance of the distillation reward effectively stabilizes the training process and subsequently unlocks the student's reasoning potential. Furthermore, TOP-D firmly dominates the RLVR baselines, exceeding the highly competitive DAPO by 17.5\% absolute accuracy on the same benchmark, demonstrating the superiority of stabilized TOP-D over the traditional RLVR baselines.

When scaling down to the smaller Qwen3-1.7B-Base student, standard OPD suffers from severe performance degradation, noticeably falling behind the RLVR baselines. This fragility highlights the vulnerability of the unbounded logarithmic probability ratio when the capacity gap between the teacher and the student is severe. In stark contrast, TOP-D consistently maintains its substantial advantage. For the 1.7B student, TOP-D pushes the AIME24 accuracy past the 20\% threshold and consistently surpasses standard OPD by over 10 percentage points on both AIME24 and AIME25.

\subsection{Ablation Study}\label{Ablation Study}
To deeply understand the mechanisms driving TOP-D's success, we isolate the impact of its core structural designs. For this analysis, we utilize the Qwen3-1.7B-Base student model supervised by the Qwen3-30B-A3B-Instruct-2507 teacher. Figure \ref{ablation curves} illustrates the learning curves of TOP-D compared to its ablated variants and the standard OPD baseline throughout the entire model training process.

By setting the interpolation coefficient to $\alpha=1.0$, we remove the external proximal teacher and revert the objective to the unbounded standard OPD reward. As shown in the training curves, this ablation causes highly unstable training dynamics and a significant degradation in overall performance. On the other hand, by disabling the internal trust region iterations (denoted as w/o off-policy), we force the optimization into a strictly on-policy regime. This leads to a drastic drop in sample efficiency, requiring significantly more training steps to converge and climb the performance curve. These empirical results clearly demonstrate that both the \textit{external proximal teacher} and the \textit{internal trust region iterations} are indispensable components for achieving optimization stability and rapid performance improvement. Furthermore, we evaluate the sensitivity of TOP-D to the interpolation coefficient by varying $\alpha \in \{0.1, 0.2, 0.3\}$. We empirically observe that the training stability, convergence speed, and final performance exhibit marginal differences across this interval. This demonstrates that TOP-D is highly robust to the choice of $\alpha$, functioning as a reliable solution without the need for exhaustive and computationally expensive hyperparameter optimization efforts.

\begin{figure*}[!t]
	\centering
	\includegraphics[scale=0.5]{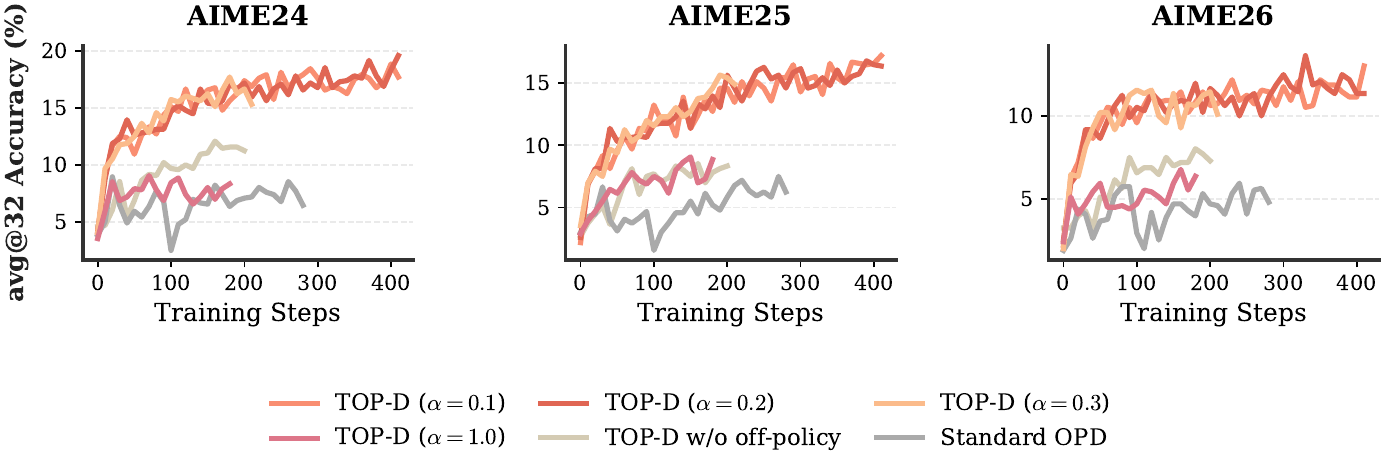}
	\caption{Learning curves for the ablation study on the Qwen3-1.7B-Base. We isolate the effects of the external proximal teacher ($\alpha=1.0$) and the internal trust region iterations (w/o off-policy).}
	\label{ablation curves}
\end{figure*}

\section{Conclusion}
In this paper, we introduce Trust Region Policy Distillation (TOP-D), a principled post-training paradigm resolving the optimization instability and sample inefficiency of standard On-Policy Distillation (OPD). Theoretically, TOP-D bounds gradient variance via an external proximal teacher and guarantees monotonic improvement with safe off-policy data reuse through internal trust region iterations. Empirically, it achieves massive performance gains on mathematical reasoning benchmarks, decisively outperforming standard OPD and competitive RLVR baselines across multiple model scales. By achieving these advancements without any additional computational overhead, TOP-D establishes a highly reliable, sample-efficient, and robust paradigm for aligning foundation models.

\newpage
\begin{ack}
	We would like to thank the anonymous reviewers for their comments. This work was supported by the National Natural Science Foundation of China under Grant No. 62506317 and The Hong Kong University of Science and Technology (Guangzhou) Kunpeng\&Ascend Center of Cultivation.
\end{ack}

\nocite{zhao2026late, gu2026mano}
\bibliography{references}
\bibliographystyle{plainnat}

%%%%%%%%%%%%%%%%%%%%%%%%%%%%%%%%%%%%%%%%%%%%%%%%%%%%%%%%%%%%
\appendix

\newpage
\section{More Results}

\begin{table}[!h]
	\centering
	\renewcommand{\arraystretch}{1.1}
	\resizebox{\linewidth}{!}{
		\begin{tabular}{llc ccc ccc}
			\toprule
			\multirow{4}{*}[-0.8em]{\textbf{Student Model}} & \multirow{4}{*}[-0.8em]{\textbf{Paradigm}} & \multirow{4}{*}[-0.8em]{\textbf{Method}} & \multicolumn{6}{c}{\textbf{Teacher Model}} \\
			\cmidrule(lr){4-9}
			& & & \multicolumn{6}{c}{\textit{\textbf{Qwen3-30B-A3B-Instruct-2507}}} \\
			\cmidrule(lr){4-9}
			& & & \multicolumn{4}{c}{\textbf{avg@32}} & \multicolumn{2}{c}{\textbf{avg@8}} \\
			\cmidrule(lr){4-7} \cmidrule(lr){8-9}
			& & & \textbf{AIME24} & \textbf{AIME25} & \textbf{AIME26} & \textbf{AMC23} & \textbf{MATH-500} & \textbf{Olympiad} \\
			\midrule
			
			\multirow{5}{*}{\textit{\textbf{Qwen3-1.7B-Base}}} 
			& \textbf{Base} & - & 3.44 & 1.98 & 1.88 & 27.27 & 56.80 & 24.00 \\
			\cmidrule(l){2-9}
			& \multirow{2}{*}{\textbf{RLVR}} & GRPO & 10.52 & 7.19 & 4.69 & 35.55 & 54.73 & 18.47 \\
			& & DAPO & 12.29 & 11.56 & 8.54 & 37.50 & 61.88 & 26.63 \\
			\cmidrule(l){2-9}
			& \multirow{2}{*}{\textbf{Distillation}} & OPD & 8.96 & 7.50 & 5.94 & 45.70 & 71.13 & 35.83 \\
			& & \cellcolor{1}\textbf{TOP-D (Ours)} & \cellcolor{1}\textbf{20.31} {\color{2}\scriptsize\textbf{(+11.35)}} & \cellcolor{1}\textbf{17.71} {\color{2}\scriptsize\textbf{(+10.21)}} & \cellcolor{1}\textbf{13.75} {\color{2}\scriptsize\textbf{(+7.81)}} & \cellcolor{1}\textbf{54.38} {\color{2}\scriptsize\textbf{(+8.68)}} & \cellcolor{1}\textbf{77.33} {\color{2}\scriptsize\textbf{(+6.20)}} & \cellcolor{1}\textbf{43.06} {\color{2}\scriptsize\textbf{(+7.23)}} \\
			\bottomrule
		\end{tabular}
	}
	\caption{Performance comparison of the Qwen3-1.7B-Base across various mathematical benchmarks.}
\end{table}

\section{Related Work}
\textbf{Trust region methods.} \citet{schulman2015trust} proposed Trust Region Policy Optimization (TRPO), which extends the trust region concept to policy gradient methods and provides the first theoretical guarantee of monotonic policy improvement. Proximal Policy Optimization (PPO) \citep{schulman2017proximal} overcomes the computationally expensive second-order optimization of TRPO and has been widely adopted. More recently, Simple Policy Optimization (SPO) \citep{xie2025simple} proposed a quadratic penalty objective function, thereby combining the theoretical guarantees of TRPO with the practical efficiency of PPO smoothly.

%\citep{schulman2015trust, schulman2017proximal, xie2025simple}

\textbf{Trust region in on-policy distillation.} Recent efforts incorporate trust region mechanisms to stabilize OPD against teacher-student distribution mismatches. TrOPD \citep{xing2026trust} partitions student tokens by teacher acceptance probability, restricting reverse-KL updates to reliable regions while mitigating outlier gradients via forward-KL estimation. TRB \citep{plyusov2026trust} constrains the rollout behavior within a student-centered KL ball during early warmup, generating teacher-guided prefixes without altering the per-prefix distillation objective. Meanwhile, ERPD \citep{chen2026extreme} decouples sample and KL efficiency by first extracting supervisory signals through aggressive off-policy optimization, and then distilling them back into the base policy under explicit trust region constraints to filter out spurious policy drift.

\section*{Computational Resources}
All primary experiments, including the training of the TOP-D and the reproduction of baseline methods, were conducted on a high-performance computing cluster equipped with NVIDIA H200 GPUs. To accelerate the research iteration cycle, our standard distributed training setup utilized 4 compute nodes, with each node containing 8 H200 GPUs (amounting to 32 GPUs in total). However, we emphasize that our proposed method is highly resource-efficient and accessible. The entire TOP-D training pipeline and its reported results can be fully reproduced on a single standard 8-GPU server.

\section*{Limitations}
Although TOP-D establishes a highly reliable and sample-efficient paradigm for policy distillation, yielding massive improvements over existing baselines, our current validation is naturally bounded by computational budgets and time constraints. Specifically, our primary validations focus on student models up to the 8B parameter scale. Due to resource limitations, we have not yet investigated the scaling behavior of TOP-D on massive-scale foundation models (e.g., scaling the student model beyond 30B parameters). However, given the theoretically grounded variance-reduction properties of our proximal teacher and the guaranteed monotonic improvement from our trust region iterations, we are optimistic that the robust performance of TOP-D will seamlessly transfer to larger-scale regimes. 

Furthermore, our current empirical evaluations are constrained to a relatively short optimization window of only approximately 200 to 400 update steps (except for the RLVR baselines). Crucially, we have not observed any empirical signs of performance saturation within this timeframe. This strongly suggests that the impressive results reported in our main text are far from their upper bounds, leaving substantial room for further improvements with continued training. We leave the exploration of such extreme-scale distillation and extended training horizons as exciting avenues for future work.

\section{Technical Details}
In this section, we elaborate on the asymptotic error in Section \ref{Asymptotic Error} and monotonic improvement in Section \ref{Practical Objective} to further clarify our main theoretical results.

\subsection{Asymptotic Error}\label{Asymptotic Error}
While the factor $1/\alpha$ in Theorem \ref{convergence} seemingly suggests that a small $\alpha$ amplifies the asymptotic bound $\epsilon_\infty/\alpha$, we show that this bound is actually invariant to $\alpha$; thus, without internal trust region iterations, TOP-D shares the exact asymptotic optimization error of standard OPD. Specifically, for $\tilde{\pi}_k^* = \alpha\pi^* + (1-\alpha)\pi_k$, the initial target distance from the current student $\pi_k$ scales linearly with $\alpha$:
\begin{equation}
	d(\pi_k, \tilde{\pi}_k^*) = \mathbb{E}_{x\sim\mathcal{D}_x}\left[\lVert\pi_k(\cdot\mid x) - \tilde{\pi}_k^*(\cdot\mid x)\rVert_1\right] = \alpha \, d(\pi_k, \pi^*).
\end{equation}
Since the proximal target $\tilde{\pi}_k^*$ is $\alpha$-closer to $\pi_k$ along the interpolation path toward $\pi^*$ than in standard OPD ($\alpha=1$), the single-step optimization residual $\lVert\epsilon_k\rVert_1 = d(\pi_{k+1}, \tilde{\pi}_k^*)$ under the same single-step approximation capability scales proportionally with the initial target distance, i.e., $\lVert\epsilon_k\rVert_1 = \alpha \lVert\epsilon_k^{\mathrm{OPD}}\rVert_1$. Taking the limit supremum as $k\to\infty$ gives $\epsilon_\infty = \alpha \epsilon_\infty^{\mathrm{OPD}}$, which yields
\begin{equation}
	\limsup_{k\to\infty} d(\pi_{k+1}, \pi^*) \le \frac{\epsilon_\infty}{\alpha} = \epsilon_\infty^{\mathrm{OPD}}.
\end{equation}
Consequently, any further effort to reduce $\lVert\epsilon_k\rVert_1$ directly tightens the asymptotic bound. Dividing \eqref{decreasing sequence of KL divergences} by $\alpha > 0$ gives
\begin{equation}
	\frac{\mathbb{E}_{x\sim\mathcal{D}_x}\left[\mathcal{D}_{\mathrm{KL}}^{(0)}\right]}{\alpha} \ge \frac{\mathbb{E}_{x\sim\mathcal{D}_x}\left[\mathcal{D}_{\mathrm{KL}}^{(1)}\right]}{\alpha} \ge \dots \ge \frac{\mathbb{E}_{x\sim\mathcal{D}_x}\left[\mathcal{D}_{\mathrm{KL}}^{(n)}\right]}{\alpha}.
\end{equation}
For the output $\pi_{k+1} = \pi_k^{(n)}$, Pinsker's and Jensen's inequalities (Section \ref{Monotonic Improvement}) further imply
\begin{equation}
	\frac{\lVert\epsilon_k\rVert_1}{\alpha} = \frac{d(\pi_k^{(n)}, \tilde{\pi}_k^*)}{\alpha} \le \frac{\mathbb{E}_{x\sim\mathcal{D}_x}\left[\sqrt{2\mathcal{D}_{\mathrm{KL}}^{(n)}}\right]}{\alpha} \le \frac{\sqrt{2\mathbb{E}_{x\sim\mathcal{D}_x}\left[\mathcal{D}_{\mathrm{KL}}^{(n)}\right]}}{\alpha} \le \frac{\delta}{\alpha},
\end{equation}
showing that lower expected KL divergence tightens the Pinsker upper bound on $\lVert\epsilon_k\rVert_1/\alpha$.

\subsection{Practical Objective}\label{Practical Objective}
Building on the insights of \citet{xie2025simple}, the performance lower bound in Theorem \ref{performance lower bound} can be transformed into a practical form. Specifically, by applying importance sampling and employing the TV divergence identity, we have
\begin{equation}
	\begin{split}
		\zeta_\pi(\tilde{\pi})&=\eta(\pi) + \ell_\pi\mathbb{E}_{s \sim d^\pi_{\mathrm{norm}}(\cdot), a \sim \tilde{\pi}(\cdot \mid s)} \left[ A^\pi(s, a) \right]\\
		&=\eta(\pi) + \ell_\pi\mathbb{E}_{s \sim d^\pi_{\mathrm{norm}}(\cdot), a \sim \pi(\cdot \mid s)} \left[ \frac{\tilde{\pi}(a \mid s)}{\pi(a \mid s)}\cdot A^\pi(s, a) \right], \\
	\end{split}
\end{equation}
and
\begin{equation}
	\mathbb{E}_{s \sim d^{\pi}_{\mathrm{norm}}(\cdot)}\left[\mathcal{D}_{\mathrm{TV}}(\tilde{\pi}(\cdot \mid s),\pi(\cdot \mid s))\right]=\frac{1}{2}\mathbb{E}_{s \sim d^\pi_{\mathrm{norm}}(\cdot), a \sim \pi(\cdot \mid s)}\left[\left|\frac{\tilde{\pi}(a \mid s)}{\pi(a \mid s)}-1\right|\right].
\end{equation}
We now arrive at the following constrained optimization problem:
\begin{equation}
	\max_{\tilde{\pi}}\enspace\mathbb{E}_{s \sim d^\pi_{\mathrm{norm}}(\cdot), a \sim \pi(\cdot \mid s)} \left[ \frac{\tilde{\pi}(a \mid s)}{\pi(a \mid s)}\cdot A^\pi(s, a) \right]\enspace\mathrm{s.t.}\enspace\mathbb{E}_{s \sim d^\pi_{\mathrm{norm}}(\cdot), a \sim \pi(\cdot \mid s)}\left[\left|\frac{\tilde{\pi}(a \mid s)}{\pi(a \mid s)}-1\right|\right]\leq\epsilon.
\end{equation}
To efficiently solve this constrained optimization problem using first-order methods, PPO employs a heuristic ratio-clipping mechanism to approximately enforce the TV divergence trust region constraints, yielding the practical objective in \eqref{ppo objective}, which we adapt for our internal trust region iterations.

As described by \citet{xie2025simple}, PPO's hard clipping zeros the gradients of certain data points, leading to a lack of corrective gradients to consistently enforce trust region constraints. Therefore, Simple Policy Optimization (SPO) is designed to make the policy optimization process corrective and stable.

\newpage
\section{Hyperparameters}\label{Hyperparameters}
\begin{table}[!h]
	\centering
	\renewcommand{\arraystretch}{1.2}
	\resizebox{1.0\linewidth}{!}{
		\begin{tabular}{l c c c c}
			\toprule
			\textbf{Hyperparameter} & \textbf{GRPO} \citep{shao2024deepseekmath} & \textbf{DAPO} \citep{yu2025dapo} & \textbf{OPD} \citep{lu2025onpolicydistillation} & \textbf{TOP-D (Ours)} \\
			\midrule
			
			\rowcolor{0}\multicolumn{5}{l}{\textit{\textbf{General}}} \\
			Max prompt length & 2048 & 2048 & 2048 & 2048 \\
			
			\midrule
			\rowcolor{0}\multicolumn{5}{l}{\textit{\textbf{Optimization}}} \\
			Global batch size & 512 & 512 & 512 & 512 \\
			Mini-batch size & 32 & 32 & 512 & 32 \\
			Mini-batches & 16 & 16 & 1 & 16 \\
			Group size & 8 & 8 & 8 & 8 \\
			Off-policy epoch & 1 & 1 & - & 1 \\
			Optimizer & AdamW & AdamW & AdamW & AdamW \\
			Learning rate & 1e-6 & 1e-6 & 1e-6 & 1e-6 \\
			Warmup steps & 10 & 10 & - & - \\
			
			\midrule
			\rowcolor{0}\multicolumn{5}{l}{\textit{\textbf{Sampling}}} \\
			Rollout max response length & 16384 & 16384 & 16384 & 16384 \\
			Rollout temperature & 1.0 & 1.0 & 1.0 & 1.0 \\
			Rollout top-$p$ & 1.0 & 1.0 & 1.0 & 1.0 \\
			Validation max response length & 16384 & 16384 & 16384 & 16384 \\
			Validation temperature & 1.0 & 1.0 & 1.0 & 1.0 \\
			Validation top-$p$ & 0.7 & 0.7 & 0.7 & 0.7 \\
			
			\midrule
			\rowcolor{0}\multicolumn{5}{l}{\textit{\textbf{Algorithm-specific}}} \\
			$\epsilon_{\mathrm{low}}$ & 0.2 & 0.2 & - & 0.2 \\
			$\epsilon_{\mathrm{high}}$ & 0.2 & 0.28 & - & 0.2 \\
			Interpolation $\alpha$ & - & - & - & 0.1 / 0.2 \\
			
			\bottomrule
		\end{tabular}
	}
	\caption{Detailed hyperparameters for different methods.}
	\label{hyperparameters}
\end{table}

\newpage
\section{Proofs}

\subsection{Proof of Theorem \ref{bounded variance}} \label{proof of bounded variance}
\textbf{Theorem \ref{bounded variance}.} \textit{
	Under Assumption \ref{assumption 1}, for any interpolation coefficient $\alpha\in(0,1)$, the variance of the TOP-D gradient estimator $\tilde{g}_k$ does not diverge across the entire probability space. Instead, it is strictly bounded by a uniform absolute limit:
	\begin{equation}
		\operatorname{tr}(\mathrm{Var}(\tilde{g}_k))\leq M^2\left(\log^2(1-\alpha)+C^*\alpha\right),
	\end{equation}
	where $C^*$ is a universal mathematical constant independent of the vocabulary size.
}
\begin{proof}
	For any square-integrable random vector $g$, $\operatorname{tr}(\mathrm{Var}(g))=\mathbb{E}[\lVert g\rVert^2]-\lVert\mathbb{E}[g]\rVert^2\leq\mathbb{E}[\lVert g\rVert^2]$. Thus, a uniform finite bound on $\mathbb{E}[\lVert\tilde{g}_k\rVert^2]$ also establishes the required square integrability.
	
	The TOP-D immediate reward is $\tilde{r}_k=\log(\alpha\rho_k+1-\alpha)$. Fix a nonterminal context $(x,y_{<k})$, and write $p_v=\pi_\theta(v\mid x,y_{<k})$ and $q_v=\pi^*(v\mid x,y_{<k})$. For full-vocabulary softmax policies, $p_v,q_v\in(0,1)$, and we have
	\begin{equation}
		\mathbb{E}_{y_k \sim \pi_\theta(\cdot \mid x, y_{<k})}\left[\tilde{r}_k^2\right] = \sum_{v\in\mathcal{V}}p_v\log^2\left(\alpha\frac{q_v}{p_v}+1-\alpha\right).
	\end{equation}
	For one token, write $p=p_v\in(0,1)$ and $q=q_v\in(0,1)$. Finite logits give strictly positive softmax probabilities over the full vocabulary. Define
	\begin{equation}
		h(p, q) = p \log^2\left(\alpha \frac{q}{p} + 1 - \alpha\right).
	\end{equation}
	We split the vocabulary into $\mathcal{V}^-=\{v\in\mathcal{V}:q_v<p_v\}$ and $\mathcal{V}^+=\{v\in\mathcal{V}:q_v\geq p_v\}$, retaining the probability weights in both regions.
	
	\textbf{Case 1: The negative reward region $q < p$.}
	
	Here $0<\frac{q}{p}<1$, so $1-\alpha<\alpha\frac{q}{p}+1-\alpha<1$. The lower endpoint is approached as $q\to0^+$ for fixed $p>0$, and the logarithm is negative throughout this interval. Its squared magnitude is therefore bounded by the endpoint limit, giving
	\begin{equation}
		\log^2\left(\alpha \frac{q}{p} + 1 - \alpha\right) \leq \log^2(1 - \alpha).
	\end{equation}
	Multiplying by $p$ and summing over the negative region gives
	\begin{equation}
		\sum_{v\in\mathcal{V}^-}h(p_v,q_v)\leq\log^2(1-\alpha)\sum_{v\in\mathcal{V}^-}p_v\leq\log^2(1-\alpha).
	\end{equation}
	
	\textbf{Case 2: The nonnegative reward region $q \geq p$.}
	
	Here $q/p\geq1$, and hence $1\leq\alpha q/p+1-\alpha\leq\alpha q/p+1$. Both logarithms are nonnegative, so replacing the inner term by the larger one preserves the squared inequality. Crucially, we retain the teacher probability $q$:
	\begin{equation}
		h(p,q)\leq p\log^2\left(\alpha\frac{q}{p}+1\right).
	\end{equation}
	Set $u=q/p\in[1,\infty)$, so that $p=q/u$. This substitution separates the teacher's probability mass from the likelihood ratio, rather than bounding each token by an unweighted constant. The factor $q$ must remain until the vocabulary sum is taken. In particular,
	\begin{equation}
		h(p,q)\leq\frac{q}{u}\log^2(\alpha u+1)=q\,f(u).
	\end{equation}
	The remaining task is to bound the same scalar envelope uniformly over all $u\geq1$:
	\begin{equation}
		f(u) = \frac{\log^2(\alpha u + 1)}{u}.
	\end{equation}
	To find the global maximum of $f(u)$ on $[1,\infty)$, we compute its derivative with respect to $u$:
	\begin{equation}
		\begin{split}
			f'(u)&=\frac{2\alpha\log(1+\alpha u)}{u(1+\alpha u)}-\frac{\log^2(1+\alpha u)}{u^2}\\
			&=\frac{\log(1+\alpha u)}{u^2}\left(\frac{2\alpha u}{1+\alpha u}-\log(1+\alpha u)\right).\\
		\end{split}
	\end{equation}
	Setting $f'(u) = 0$ identifies the stationary points. Since $u \geq 1$ and $\alpha \in (0, 1)$, the term $\log(\alpha u + 1) > 0$ and can be safely factored out. The condition simplifies to
	\begin{equation}
		\frac{2\alpha u}{\alpha u + 1} = \log(\alpha u + 1).
	\end{equation}
	We further introduce a variable substitution $t = \alpha u + 1$. Because $\alpha u > 0$, we have $t > 1$. Since $\alpha u = t - 1$, the condition transforms into a pure transcendental equation independent of $\alpha$:
	\begin{equation}
		2\left(1-\frac{1}{t}\right) = \log t \implies \log t - 2\left(1 - \frac{1}{t}\right) = 0.
	\end{equation}
	Let $w(t) = \log t - 2\left(1 - \frac{1}{t}\right)$. We analyze its roots for $t > 1$:
	\begin{enumerate}
		\item The derivative is $w'(t) = \frac{1}{t} - \frac{2}{t^2} = \frac{t - 2}{t^2}$.
		\item For $t \in (1, 2)$, $w'(t) < 0$, meaning $w(t)$ strictly decreases. Since $\lim_{t \to 1^+} w(t) = 0$, $w(t) < 0$ for all $t \in (1, 2)$.
		\item For $t > 2$, $w'(t) > 0$, meaning $w(t)$ strictly increases.
		\item Since $w(2)=\log2-1 < 0$ and $w(e^2)=\frac{2}{e^2}>0$, by the Intermediate Value Theorem, there exists exactly one unique root $t^*$ in the interval $(2, e^2)$.
	\end{enumerate}
	Since $f'(u)=-\log(1+\alpha u)w(1+\alpha u)/u^2$, its sign is positive before $u^*=(t^*-1)/\alpha$ and negative afterwards. Also, $1+\alpha<2<t^*$ ensures $u^*>1$, so the maximizer lies inside the domain for every admissible $\alpha$. Substituting this unique global maximizer gives
	\begin{equation}
		f(u)\leq f(u^*)=\frac{\log^2 t^*}{t^* - 1}\cdot\alpha.
	\end{equation}
	Let $C^*=\frac{\log^2 t^*}{t^*-1}=\sup_{z>0}\frac{\log^2(1+z)}{z}$. Then $h(p,q)\leq qf(u)\leq C^*\alpha q$ for $q\geq p$.
	
	Summing the probability-weighted bound over the nonnegative region uses at most the teacher's total probability mass, rather than a separate constant for each vocabulary token:
	\begin{equation}
		\sum_{v\in\mathcal{V}^+}h(p_v,q_v)\leq C^*\alpha\sum_{v\in\mathcal{V}^+}q_v\leq C^*\alpha.
	\end{equation}
	
	The two regions partition the vocabulary. Combining their bounds controls the whole conditional expectation, regardless of the context $(x,y_{<k})$ or vocabulary size. No multiplication by the number of tokens is needed:
	\begin{equation}
		\begin{split}
			\mathbb{E}_{y_k \sim \pi_\theta(\cdot \mid x, y_{<k})}\left[\tilde{r}_k^2\right]&=\sum_{v\in\mathcal{V}^-}h(p_v,q_v)+\sum_{v\in\mathcal{V}^+}h(p_v,q_v)\\
			&\leq\log^2(1-\alpha)+C^*\alpha.
		\end{split}
	\end{equation}
	Now, we move to the TOP-D gradient estimator $\tilde{g}_k = \nabla_\theta \log \pi_\theta(y_k \mid x, y_{<k}) \cdot \tilde{r}_k$. According to Assumption \ref{assumption 1}, we have
	\begin{equation}
		\lVert\tilde{g}_k\rVert^2 = \lVert\nabla_\theta \log \pi_\theta(y_k \mid x, y_{<k})\rVert^2 \cdot \tilde{r}_k^2 \leq M^2 \tilde{r}_k^2.
	\end{equation}
	Set the reward and estimator to zero after termination. Averaging the context-uniform second-moment bound over prompts, on-policy prefixes, and tokens gives
	\begin{equation}
		\begin{split}
			\mathbb{E}\left[\lVert\tilde{g}_k\rVert^2\right] &= \mathbb{E}_{x \sim \mathcal{D}_x, y_{<k} \sim \pi_\theta(\cdot \mid x)} \left[ \mathbb{E}_{y_k \sim \pi_\theta(\cdot \mid x, y_{<k})} \left[ \lVert\tilde{g}_k\rVert^2 \right] \right] \\
			&\leq\mathbb{E}_{x \sim \mathcal{D}_x, y_{<k} \sim \pi_\theta(\cdot \mid x)} \left[ M^2 \mathbb{E}_{y_k \sim \pi_\theta(\cdot \mid x, y_{<k})}\left[\tilde{r}_k^2\right]\right]\\
			&\leq\mathbb{E}_{x \sim \mathcal{D}_x, y_{<k} \sim \pi_\theta(\cdot \mid x)} \left[ M^2\left(\log^2(1-\alpha)+C^*\alpha\right)\right]\\
			&=M^2\left(\log^2(1-\alpha)+C^*\alpha\right).\\
		\end{split}
	\end{equation}
	Finally, we return to our variance inequality established at the beginning of the proof:
	\begin{equation}
		\operatorname{tr}(\mathrm{Var}(\tilde{g}_k)) \leq \mathbb{E}\left[\lVert\tilde{g}_k\rVert^2\right] \leq M^2\left(\log^2(1-\alpha)+C^*\alpha\right).
	\end{equation}
	Since $M$, $\alpha$, and $C^*$ are all finite, the gradient covariance trace is absolutely bounded. The proof is now complete.
\end{proof}

\subsection{Proof of Theorem \ref{convergence}} \label{proof of convergence}
\textbf{Theorem \ref{convergence}.} \textit{
	For any initial student policy $\pi_0\in\Pi$, consider the iterative update process $\pi_{k+1}=\mathcal{T}(\pi_k)+\epsilon_k$, where $\epsilon_k$ accounts for the practical optimization error (e.g., function approximation and optimization noise) at the $k$-th global step, where $k\geq0$. Then, the following bound holds:
	\begin{equation}
		d(\pi_{k+1},\pi^*)\leq(1-\alpha)^{k+1}d(\pi_0,\pi^*)+\sum_{i=0}^{k}(1-\alpha)^{k-i}\lVert\epsilon_i\rVert_1.
	\end{equation}
	Furthermore, assuming the optimization process eventually stabilizes such that the asymptotic error is bounded by $\limsup_{k\rightarrow \infty}\lVert\epsilon_{k}\rVert_1\leq \epsilon_{\infty}$, we have
	\begin{equation}
		\limsup_{k \rightarrow \infty} d(\pi_{k+1}, \pi^*) \leq \frac{\epsilon_{\infty}}{\alpha}.
	\end{equation}
}
\begin{proof}
	We aim to bound the expected $L_1$ distance between the student policy $\pi_{k+1}$ and the optimal target teacher policy $\pi^*$.
	
	\textbf{Step 1: Establishing the single-step recurrence relation.}
	
	According to the iterative update process defined in the theorem, for any global step $k \geq 0$, the updated policy is given by
	\begin{equation}
		\pi_{k+1} = \mathcal{T}(\pi_k) + \epsilon_k,
	\end{equation}
	where $\mathcal{T}:\Pi\to\Pi$ is the \textit{proximal teacher operator} defined in \eqref{proximal teacher operator}. By utilizing the triangle inequality of the expected $L_1$ distance metric $d(\cdot, \cdot)$, we decouple the distance between the updated policy $\pi_{k+1}$ and the target policy $\pi^*$ into two components:
	\begin{equation}\label{triangle inequality}
		d(\pi_{k+1}, \pi^*) \leq d(\mathcal{T}(\pi_k), \pi^*) + d(\pi_{k+1}, \mathcal{T}(\pi_k)).
	\end{equation}
	We analyze the second term on the right-hand side first. As defined in the main text, the proximal teacher is explicitly given by $\tilde{\pi}_k^*=\mathcal{T}(\pi_k)=\alpha\pi^*+(1-\alpha)\pi_k$. Therefore, the distance between the updated policy $\pi_{k+1}$ and the proximal teacher $\mathcal{T}(\pi_k)$ precisely corresponds to the expected $L_1$ norm of the optimization error at step $k$:
	\begin{equation}
		d(\pi_{k+1}, \mathcal{T}(\pi_k)) = d(\pi_{k+1}, \tilde{\pi}_k^*) = \lVert\epsilon_k\rVert_1.
	\end{equation}
	For the first term, we substitute the explicit definition of the proximal operator $\mathcal{T}(\pi_k)$. Utilizing the absolute homogeneity and linearity of the expected $L_1$ norm over $\mathcal{D}_x$, we obtain
	\begin{equation}
		\begin{split}
			d(\mathcal{T}(\pi_k), \pi^*)&=\mathbb{E}_{x\sim\mathcal{D}_x}\left[\lVert\alpha\pi^*(\cdot\mid x)+(1-\alpha)\pi_k(\cdot\mid x)-\pi^*(\cdot\mid x)\rVert_1\right]\\
			&=(1-\alpha)\mathbb{E}_{x\sim\mathcal{D}_x}\left[\lVert\pi_k(\cdot\mid x)-\pi^*(\cdot\mid x)\rVert_1\right]\\
			&=(1-\alpha)d(\pi_k,\pi^*).
		\end{split}
	\end{equation}
	Substituting both decoupled terms back into the triangle inequality \eqref{triangle inequality}, we establish the strict single-step recurrence relation for any step:
	\begin{equation}
		d(\pi_{k+1}, \pi^*) \leq(1-\alpha)d(\pi_k,\pi^*)+\lVert\epsilon_k\rVert_1.
	\end{equation}
	
	\textbf{Step 2: Forward unrolling of the recurrence relation.}
	
	To derive the explicit global bound for step $k+1$, we recursively unroll this single-step bound forward, starting from the initial step $k=0$.
	
	For $k=0$, the single-step recurrence gives
	\begin{equation}
		d(\pi_{1}, \pi^*) \leq(1-\alpha)d(\pi_0,\pi^*)+\lVert\epsilon_0\rVert_1.
	\end{equation}
	
	For $k=1$, we apply the recurrence and substitute the previously established bound for $d(\pi_{1}, \pi^*)$:
	\begin{equation}
		\begin{split}
			d(\pi_{2}, \pi^*) &\leq(1-\alpha)d(\pi_1,\pi^*)+\lVert\epsilon_1\rVert_1\\
			&\leq(1-\alpha)^2d(\pi_0,\pi^*)+(1-\alpha)\lVert\epsilon_0\rVert_1+\lVert\epsilon_1\rVert_1.
		\end{split}
	\end{equation}
	
	For $k=2$, we repeat the process by substituting the bound for $d(\pi_{2}, \pi^*)$:
	\begin{equation}
		\begin{split}
			d(\pi_{3}, \pi^*) &\leq(1-\alpha)d(\pi_2,\pi^*)+\lVert\epsilon_2\rVert_1\\
			&\leq(1-\alpha)^3d(\pi_0,\pi^*)+(1-\alpha)^2\lVert\epsilon_0\rVert_1+(1-\alpha)\lVert\epsilon_1\rVert_1+\lVert\epsilon_2\rVert_1.
		\end{split}
	\end{equation}
	Through rigorous algebraic expansion by repeating this substitution multiple times, a clear geometric pattern emerges. Accumulating these terms yields the closed-form finite-step global bound:
	\begin{equation}
		d(\pi_{k+1}, \pi^*)\leq(1-\alpha)^{k+1}d(\pi_0,\pi^*)+\sum_{i=0}^{k}(1-\alpha)^{k-i}\lVert\epsilon_i\rVert_1.
	\end{equation}
	
	\textbf{Step 3: Deriving the asymptotic convergence bound.}
	
	To determine the asymptotic behavior, we apply the limit supremum ($\limsup_{k \rightarrow \infty}$) to both sides of the inequality:
	\begin{equation}
		\limsup_{k \rightarrow \infty}d(\pi_{k+1}, \pi^*)\leq\limsup_{k \rightarrow \infty}\left\{(1-\alpha)^{k+1}d(\pi_0,\pi^*)+\sum_{i=0}^{k}(1-\alpha)^{k-i}\lVert\epsilon_i\rVert_1\right\}.
	\end{equation}
	First, the interpolation coefficient satisfies $\alpha \in (0, 1)$. Consequently, the term associated with the initial distance decays exponentially to zero:
	\begin{equation}
		\limsup_{k \rightarrow \infty}(1-\alpha)^{k+1}d(\pi_0,\pi^*)=0.
	\end{equation}
	Next, we analyze the accumulated error summation term. We are given the assumption that the asymptotic optimization error is bounded, i.e., $\limsup_{k\rightarrow \infty}\lVert\epsilon_{k}\rVert_1\leq \epsilon_{\infty}$. By the mathematical definition of the limit supremum, for any arbitrarily small constant $\delta > 0$, there exists a sufficiently large integer $N > 0$ such that for all $i \geq N$, the condition $\lVert\epsilon_i\rVert_1 \leq \epsilon_{\infty} + \delta$ strictly holds.
	
	Let $k\to\infty$; we now split the summation into two parts at this critical index $N$:
	\begin{equation}
		\sum_{i=0}^{k}(1-\alpha)^{k-i}\lVert\epsilon_i\rVert_1=\sum_{i=0}^{N-1}(1-\alpha)^{k-i}\lVert\epsilon_i\rVert_1+\sum_{i=N}^{k}(1-\alpha)^{k-i}\lVert\epsilon_i\rVert_1.
	\end{equation}
	As $k \to \infty$, the first term vanishes to zero. For the second term, we safely substitute the uniform bound $\lVert\epsilon_i\rVert_1 \leq \epsilon_{\infty} + \delta$ for all terms where $i \geq N$:
	\begin{equation}
		\sum_{i=N}^{k}(1-\alpha)^{k-i}\lVert\epsilon_i\rVert_1\leq(\epsilon_{\infty}+\delta)\sum_{i=N}^{k}(1-\alpha)^{k-i}.
	\end{equation}
	We introduce a change of variable by letting $j = k - i$. When $i = k$, $j=0$. When $i=N$, $j = k - N$. This transforms the summation into a standard geometric series:
	\begin{equation}
		(\epsilon_{\infty}+\delta)\sum_{i=N}^{k}(1-\alpha)^{k-i}=(\epsilon_{\infty}+\delta)\sum_{j=0}^{k-N}(1-\alpha)^{j}.
	\end{equation}
	Now we have
	\begin{equation}
		\limsup_{k\to\infty}\sum_{i=N}^{k}(1-\alpha)^{k-i}\lVert\epsilon_i\rVert_1\leq(\epsilon_{\infty}+\delta)\lim_{k\to\infty}\sum_{j=0}^{k-N}(1-\alpha)^j=(\epsilon_{\infty}+\delta)\sum_{j=0}^{\infty}(1-\alpha)^j=\frac{\epsilon_{\infty}+\delta}{\alpha}.
	\end{equation}
	Because the choice of $\delta > 0$ is arbitrary and can be chosen to be infinitesimally small ($\delta \to 0$), the limit supremum of the entire sequence is strictly bounded by $\epsilon_{\infty}/\alpha$. Combining this bound with the vanishing first term, we have
	\begin{equation}
		\limsup_{k \rightarrow \infty}d(\pi_{k+1}, \pi^*)\leq\frac{\epsilon_{\infty}}{\alpha}.
	\end{equation}
	The proof is now complete.
\end{proof}

\subsection{Proof of Theorem \ref{performance lower bound}} \label{proof of performance lower bound}
\textbf{Theorem \ref{performance lower bound}.} \textit{
	Given Assumption \ref{assumption 2} and Definition \ref{normalized state distribution}, for any two policies, $\tilde{\pi}$ and $\pi$, the following bound holds:
	\begin{equation}
		\eta(\tilde{\pi})\geq\zeta_\pi(\tilde{\pi}) - 2\xi T_{\max}\ell_\pi\mathbb{E}_{s \sim d^{\pi}_{\mathrm{norm}}(\cdot)}\left[\mathcal{D}_{\mathrm{TV}}(\tilde{\pi}(\cdot \mid s),\pi(\cdot \mid s))\right],
	\end{equation}
	where $\zeta_\pi(\tilde{\pi})=\eta(\pi) + \ell_\pi\mathbb{E}_{s \sim d^\pi_{\mathrm{norm}}(\cdot), a \sim \tilde{\pi}(\cdot \mid s)} \left[ A^\pi(s, a) \right]$, and $\xi=\max_{s}\lvert\mathbb{E}_{a\sim\tilde{\pi}(\cdot\mid s)}\left[A^\pi(s,a)\right]\rvert$.
}
\begin{proof}
	To establish the monotonic improvement lower bound, we begin by bounding the absolute error between the true expected return $\eta(\tilde{\pi})$ and our constructed surrogate objective $\zeta_\pi(\tilde{\pi})$.
	
	First, we express the unnormalized performance difference:
	\begin{equation}
		\eta(\tilde{\pi}) = \eta(\pi)+\sum_{s \neq s_\bot} d^{\tilde{\pi}}(s) \sum_{a \in \mathcal{V}} \tilde{\pi}(a \mid s) A^\pi(s, a).
	\end{equation}
	$\zeta_\pi(\tilde{\pi})$ is anchored to the base policy's normalized state distribution. By applying Definition \ref{normalized state distribution}, where $d^\pi(s) = \ell_\pi d^\pi_{\mathrm{norm}}(s)$, we can unroll $\zeta_\pi(\tilde{\pi})$ back into the unnormalized measure space:
	\begin{equation}
		\zeta_\pi(\tilde{\pi}) = \eta(\pi) + \sum_{s \neq s_\bot} d^\pi(s) \sum_{a \in \mathcal{V}} \tilde{\pi}(a \mid s) A^\pi(s, a).
	\end{equation}
	We isolate the objective approximation error by subtracting $\zeta_\pi(\tilde{\pi})$ from $\eta(\tilde{\pi})$. By applying Hölder's Inequality, we decouple the global state distribution shift from the local expected advantage drift:
	\begin{equation}\label{decoupled objective approximation error}
		\lvert \eta(\tilde{\pi}) - \zeta_\pi(\tilde{\pi}) \rvert \leq \left( \sum_{s \neq s_\bot} \lvert d^{\tilde{\pi}}(s) - d^\pi(s) \rvert \right) \cdot \max_s \left\lvert \sum_{a \in \mathcal{V}} \tilde{\pi}(a \mid s) A^\pi(s, a) \right\rvert.
	\end{equation}
	The second decoupled term is precisely defined as our constant $\xi=\max_{s}\lvert\mathbb{E}_{a\sim\tilde{\pi}(\cdot\mid s)}\left[A^\pi(s,a)\right]\rvert$.
	
	Next, we rigorously bound the $L_1$ norm of the unnormalized global measure shift. Since the unnormalized measure is the sum of marginal state probabilities over time, we expand this across the maximum generation horizon $T_{\max}$ defined in Assumption \ref{assumption 2}, and apply the triangle inequality:
	\begin{equation}
		\sum_{s \neq s_\bot} \lvert d^{\tilde{\pi}}(s) - d^\pi(s) \rvert \leq \sum_{t=1}^{T_{\max}} \sum_{s \neq s_\bot} \lvert \mathbb{P}(s_t = s \mid \tilde{\pi}) - \mathbb{P}(s_t = s \mid \pi) \rvert.
	\end{equation}
	At the first generation step $t=1$, the initial state $s_1$ is strictly defined by the input prompt $x \sim \mathcal{D}_x$ and the special initial token $y_0$. Because this initial distribution is completely independent of the policy, the marginal probability distance is exactly zero:
	\begin{equation}
		\sum_{s \neq s_\bot} \lvert \mathbb{P}(s_1 = s \mid \tilde{\pi}) - \mathbb{P}(s_1 = s \mid \pi) \rvert = 0.
	\end{equation}
	For any subsequent generation step $t > 1$, we construct a telescoping sum to bound the difference in marginal probabilities. Let $\mathbb{P}_i$ denote a \textit{hybrid trajectory distribution} where the first $i$ steps are generated by the behavior policy $\pi$, and the remaining steps from $i+1$ to $t-1$ are generated by the target policy $\tilde{\pi}$. Note that $\mathbb{P}_{t-1}$ represents full generation under $\pi$, and $\mathbb{P}_0$ represents full generation under $\tilde{\pi}$. The telescoping sum bridging these two extreme distributions is given by
	\begin{equation}
		\mathbb{P}(s_t = s \mid \tilde{\pi}) - \mathbb{P}(s_t = s \mid \pi)=\sum_{i=1}^{t-1}\left(\mathbb{P}_{i-1}(s_t = s)-\mathbb{P}_i(s_t = s)\right).
	\end{equation}
	Taking the absolute value and applying the triangle inequality yields
	\begin{equation}
		\sum_{s \neq s_\bot}\lvert\mathbb{P}(s_t = s \mid \tilde{\pi}) - \mathbb{P}(s_t = s \mid \pi)\rvert\leq\sum_{i=1}^{t-1}\sum_{s \neq s_\bot}\lvert\mathbb{P}_i(s_t = s)-\mathbb{P}_{i-1}(s_t = s)\rvert.
	\end{equation}
	Observe that the adjacent hybrid distributions $\mathbb{P}_i$ and $\mathbb{P}_{i-1}$ share the identical generation history up to step $i-1$ (governed by $\pi$) and share the identical rollout strategy from step $i+1$ onward (governed by $\tilde{\pi}$). They diverge only at the specific decision step $i$.
	
	In our autoregressive generation framework, the environmental transitions are purely deterministic (appending the generated token to the prefix). Consequently, summing over all possible future states $s_t$ effectively marginalizes them out, collapsing the trajectory distance precisely to the expected $L_1$ divergence of the action distributions at the singular branching step $i$:
	\begin{equation}
		\sum_{s \neq s_\bot}\lvert\mathbb{P}_i(s_t = s) - \mathbb{P}_{i-1}(s_t = s)\rvert \leq \mathbb{E}_{s_i \sim \mathbb{P}(s_i\in\cdot\mid\pi)}\left[ \lVert\tilde{\pi}(\cdot \mid s_i) - \pi(\cdot \mid s_i)\rVert_1 \right].
	\end{equation}
	Substituting this into our bound and summing over the entire telescoping sequence from $i=1$ to $t-1$, we obtain the marginal probability distance for a specific step $t$:
	\begin{equation}
		\sum_{s \neq s_\bot}\lvert\mathbb{P}(s_t = s \mid \tilde{\pi}) - \mathbb{P}(s_t = s \mid \pi)\rvert \leq 2\sum_{i=1}^{t-1} \mathbb{E}_{s_i \sim \mathbb{P}(s_i\in\cdot\mid\pi)}\left[\mathcal{D}_{\mathrm{TV}}(\tilde{\pi}(\cdot \mid s_i), \pi(\cdot \mid s_i))\right].
	\end{equation}
	As established earlier, the initial state distribution at $t=1$ is strictly fixed by the prompt and the special initial token, yielding zero divergence. Therefore, we sum this temporal bound over the valid response steps from $t=2$ to the maximum generation horizon $T_{\max}$ to measure the global shift:
	\begin{equation}
		\sum_{s \neq s_\bot} \lvert d^{\tilde{\pi}}(s) - d^\pi(s) \rvert \leq 2\sum_{t=2}^{T_{\max}}\sum_{i=1}^{t-1} \mathbb{E}_{s_i \sim \mathbb{P}(s_i\in\cdot\mid\pi)}\left[\mathcal{D}_{\mathrm{TV}}(\tilde{\pi}(\cdot \mid s_i), \pi(\cdot \mid s_i))\right].
	\end{equation}
	To efficiently decouple this complex temporal accumulation, we rearrange the nested summations. By exchanging the order of summation between the time steps, we observe that the expected divergence at any step $i$ is repeatedly accumulated exactly $(T_{\max} - i)$ times:
	\begin{equation}
		\begin{split}
			&\sum_{t=2}^{T_{\max}}\sum_{i=1}^{t-1} \mathbb{E}_{s_i \sim \mathbb{P}(s_i\in\cdot\mid\pi)}\left[\mathcal{D}_{\mathrm{TV}}(\tilde{\pi}(\cdot \mid s_i), \pi(\cdot \mid s_i))\right]\\ =& \sum_{i=1}^{T_{\max}-1} (T_{\max} - i) \mathbb{E}_{s_i \sim \mathbb{P}(s_i\in\cdot\mid\pi)}\left[\mathcal{D}_{\mathrm{TV}}(\tilde{\pi}(\cdot \mid s_i), \pi(\cdot \mid s_i))\right]\\
			\leq& T_{\max} \sum_{i=1}^{T_{\max}} \sum_{s \neq s_\bot} \mathbb{P}(s_i = s \mid \pi) \mathcal{D}_{\mathrm{TV}}(\tilde{\pi}(\cdot \mid s), \pi(\cdot \mid s))\\
			=& T_{\max} \sum_{s \neq s_\bot} \mathcal{D}_{\mathrm{TV}}(\tilde{\pi}(\cdot \mid s), \pi(\cdot \mid s)) \left( \sum_{i=1}^{T_{\max}} \mathbb{P}(s_i = s \mid \pi) \right).\\
		\end{split}
	\end{equation}
	Here, we relax this sum by bounding the varying coefficient $(T_{\max} - i)$ with the uniform maximum horizon $T_{\max}$ and expand the expectation into a summation over specific states.
	
	Since the unnormalized state measure $d^\pi(s)$ intrinsically aggregates the expected distribution across all steps by definition, i.e., $d^\pi(s) = \sum_{t=1}^{\infty}\mathbb{P}(s_t = s \mid \pi) \geq \sum_{i=1}^{T_{\max}} \mathbb{P}(s_i = s \mid \pi)$, we can concisely upper-bound the relaxed sequence, so
	\begin{equation}
		\sum_{s \neq s_\bot} \lvert d^{\tilde{\pi}}(s) - d^\pi(s) \rvert \leq 2T_{\max} \sum_{s \neq s_\bot} d^\pi(s)\mathcal{D}_{\mathrm{TV}}(\tilde{\pi}(\cdot \mid s), \pi(\cdot \mid s)).
	\end{equation}
	Recalling Definition \ref{normalized state distribution}, we substitute the unnormalized measure with $d^\pi(s) = \ell_\pi d^\pi_{\mathrm{norm}}(s)$, gracefully transforming the summation into a formal expectation over the normalized state distribution:
	\begin{equation}
		\sum_{s \neq s_\bot} \lvert d^{\tilde{\pi}}(s) - d^\pi(s) \rvert \leq 2T_{\max}\ell_\pi\mathbb{E}_{s \sim d^{\pi}_{\mathrm{norm}}(\cdot)}\left[\mathcal{D}_{\mathrm{TV}}(\tilde{\pi}(\cdot \mid s),\pi(\cdot \mid s))\right].
	\end{equation}
	Finally, substituting this global measure shift bound back into our decoupled objective approximation error \eqref{decoupled objective approximation error}, we obtain the absolute bound for the objective difference:
	\begin{equation}
		\lvert \eta(\tilde{\pi}) - \zeta_\pi(\tilde{\pi}) \rvert \leq 2\xi T_{\max}\ell_\pi\mathbb{E}_{s \sim d^{\pi}_{\mathrm{norm}}(\cdot)}\left[\mathcal{D}_{\mathrm{TV}}(\tilde{\pi}(\cdot \mid s),\pi(\cdot \mid s))\right],
	\end{equation}
	which implies that
	\begin{equation}
		\eta(\tilde{\pi})\geq\zeta_\pi(\tilde{\pi}) - 2\xi T_{\max}\ell_\pi\mathbb{E}_{s \sim d^{\pi}_{\mathrm{norm}}(\cdot)}\left[\mathcal{D}_{\mathrm{TV}}(\tilde{\pi}(\cdot \mid s),\pi(\cdot \mid s))\right].
	\end{equation}
	The proof is now complete.
\end{proof}

%%%%%%%%%%%%%%%%%%%%%%%%%%%%%%%%%%%%%%%%%%%%%%%%%%%%%%%%%%%%

\newpage
\section*{NeurIPS Paper Checklist}

\begin{enumerate}

\item {\bf Claims}
    \item[] Question: Do the main claims made in the abstract and introduction accurately reflect the paper's contributions and scope?
    \item[] Answer: \answerYes{} % Replace by \answerYes{}, \answerNo{}, or \answerNA{}.
    \item[] Justification: The main claims made in the abstract and introduction accurately reflect the paper's contributions and scope.

\item {\bf Limitations}
    \item[] Question: Does the paper discuss the limitations of the work performed by the authors?
    \item[] Answer: \answerYes{} % Replace by \answerYes{}, \answerNo{}, or \answerNA{}.
    \item[] Justification: The paper discusses the limitations of the work performed by the authors.

\item {\bf Theory assumptions and proofs}
    \item[] Question: For each theoretical result, does the paper provide the full set of assumptions and a complete (and correct) proof?
    \item[] Answer: \answerYes{} % Replace by \answerYes{}, \answerNo{}, or \answerNA{}.
    \item[] Justification: For each theoretical result, the paper provides the full set of assumptions and a complete (and correct) proof.

    \item {\bf Experimental result reproducibility}
    \item[] Question: Does the paper fully disclose all the information needed to reproduce the main experimental results of the paper to the extent that it affects the main claims and/or conclusions of the paper (regardless of whether the code and data are provided or not)?
    \item[] Answer: \answerYes{} % Replace by \answerYes{}, \answerNo{}, or \answerNA{}.
    \item[] Justification: The paper fully discloses all the information needed to reproduce the main experimental results of the paper.

\item {\bf Open access to data and code}
    \item[] Question: Does the paper provide open access to the data and code, with sufficient instructions to faithfully reproduce the main experimental results, as described in supplemental material?
    \item[] Answer: \answerYes{} % Replace by \answerYes{}, \answerNo{}, or \answerNA{}.
    \item[] Justification: The paper provides open access to the data and code, with sufficient instructions to faithfully reproduce the main experimental results.

\item {\bf Experimental setting/details}
    \item[] Question: Does the paper specify all the training and test details (e.g., data splits, hyperparameters, how they were chosen, type of optimizer) necessary to understand the results?
    \item[] Answer: \answerYes{} % Replace by \answerYes{}, \answerNo{}, or \answerNA{}.
    \item[] Justification: We specify all the training and test details (e.g., data splits, hyperparameters, how they were chosen, and the type of optimizer) necessary to understand the results.

\item {\bf Experiment statistical significance}
    \item[] Question: Does the paper report error bars suitably and correctly defined or other appropriate information about the statistical significance of the experiments?
    \item[] Answer: \answerYes{} % Replace by \answerYes{}, \answerNo{}, or \answerNA{}.
    \item[] Justification: We report appropriate information about the statistical significance of the experiments, including the avg@32 accuracy.

\item {\bf Experiments compute resources}
    \item[] Question: For each experiment, does the paper provide sufficient information on the computer resources (type of compute workers, memory, time of execution) needed to reproduce the experiments?
    \item[] Answer: \answerYes{} % Replace by \answerYes{}, \answerNo{}, or \answerNA{}.
    \item[] Justification: We provide sufficient information on the computer resources (type of compute workers, memory, and time of execution) required to reproduce each experiment.
    
\item {\bf Code of ethics}
    \item[] Question: Does the research conducted in the paper conform, in every respect, with the NeurIPS Code of Ethics \url{https://neurips.cc/public/EthicsGuidelines}?
    \item[] Answer: \answerYes{} % Replace by \answerYes{}, \answerNo{}, or \answerNA{}.
    \item[] Justification: We confirm that the research conducted in the paper conforms, in every respect, with the NeurIPS Code of Ethics.

\item {\bf Broader impacts}
    \item[] Question: Does the paper discuss both potential positive societal impacts and negative societal impacts of the work performed?
    \item[] Answer: \answerNA{} % Replace by \answerYes{}, \answerNo{}, or \answerNA{}.
    \item[] Justification: There is no societal impact of the work performed.
    
\item {\bf Safeguards}
    \item[] Question: Does the paper describe safeguards that have been put in place for responsible release of data or models that have a high risk for misuse (e.g., pre-trained language models, image generators, or scraped datasets)?
    \item[] Answer: \answerNA{} % Replace by \answerYes{}, \answerNo{}, or \answerNA{}.
    \item[] Justification: The paper poses no such risks.

\item {\bf Licenses for existing assets}
    \item[] Question: Are the creators or original owners of assets (e.g., code, data, models), used in the paper, properly credited and are the license and terms of use explicitly mentioned and properly respected?
    \item[] Answer: \answerYes{} % Replace by \answerYes{}, \answerNo{}, or \answerNA{}.
    \item[] Justification: We properly credit the creators or original owners of all assets (e.g., code, data, models) used in the paper, and we explicitly mention and respect their licenses and terms of use.

\item {\bf New assets}
    \item[] Question: Are new assets introduced in the paper well documented and is the documentation provided alongside the assets?
    \item[] Answer: \answerNA{} % Replace by \answerYes{}, \answerNo{}, or \answerNA{}.
    \item[] Justification: The paper does not release new assets.

\item {\bf Crowdsourcing and research with human subjects}
    \item[] Question: For crowdsourcing experiments and research with human subjects, does the paper include the full text of instructions given to participants and screenshots, if applicable, as well as details about compensation (if any)? 
    \item[] Answer: \answerNA{} % Replace by \answerYes{}, \answerNo{}, or \answerNA{}.
    \item[] Justification: The paper does not involve crowdsourcing nor research with human subjects.

\item {\bf Institutional review board (IRB) approvals or equivalent for research with human subjects}
    \item[] Question: Does the paper describe potential risks incurred by study participants, whether such risks were disclosed to the subjects, and whether Institutional Review Board (IRB) approvals (or an equivalent approval/review based on the requirements of your country or institution) were obtained?
    \item[] Answer: \answerNA{} % Replace by \answerYes{}, \answerNo{}, or \answerNA{}.
    \item[] Justification: The paper does not involve crowdsourcing nor research with human subjects.

\item {\bf Declaration of LLM usage}
    \item[] Question: Does the paper describe the usage of LLMs if it is an important, original, or non-standard component of the core methods in this research? Note that if the LLM is used only for writing, editing, or formatting purposes and does \emph{not} impact the core methodology, scientific rigor, or originality of the research, declaration is not required.
    %this research? 
    \item[] Answer: \answerNA{} % Replace by \answerYes{}, \answerNo{}, or \answerNA{}.
    \item[] Justification: The core method development in this research does not involve LLMs as any important, original, or non-standard components.

\end{enumerate}

\end{document}